\def\year{2022}\relax
\documentclass[letterpaper]{article} 
\usepackage{aaai22}  
\usepackage{times}  
\usepackage{helvet}  
\usepackage{courier}  
\usepackage[hyphens]{url}  
\usepackage{graphicx} 
\usepackage{natbib}  
\usepackage{caption} 
\DeclareCaptionStyle{ruled}{labelfont=normalfont,labelsep=colon,strut=off} 
\usepackage{url}            
\usepackage{booktabs}       
\usepackage{amsfonts}       
\usepackage{nicefrac}       
\usepackage{microtype}      
\usepackage[dvipsnames]{xcolor}
\usepackage[ruled,linesnumbered]{algorithm2e}
\usepackage{subcaption} 
\usepackage{amsthm}
\usepackage{amsmath}
\usepackage{amssymb}
\usepackage{mathtools}
\usepackage{xspace}
\usepackage{enumitem}
\setlist[itemize]{noitemsep, topsep=1pt, label=$\blacktriangleright$, leftmargin=*}
\usepackage[color=Aquamarine!20,textsize=scriptsize]{todonotes}
\usepackage{capt-of}
\usepackage{bm}

\usepackage{pifont}

\newcommand{\Real}{\mathbb{R}}
\newcommand{\HC}{\{0,1\}}
\newcommand{\fh}{\texttt{FlyHash}\xspace}
\newcommand{\sh}{\texttt{SimHash}\xspace}

\newcommand{\fbf}{\texttt{FBF}\xspace}
\newcommand{\sbfc}{\texttt{SBFC}\xspace}

\newcommand{\lr}{\texttt{LR}\xspace}
\newcommand{\mlpc}{\texttt{MLPC}\xspace}
\newcommand{\nnc}{\texttt{NNC}\xspace}
\newcommand{\fnn}{\texttt{FlyNN}\xspace}
\newcommand{\knnc}{\texttt{$k$NNC}\xspace}
\newcommand{\onennc}{\texttt{$1$NNC}\xspace}
\newcommand{\sklearn}{\texttt{scikit-learn}\xspace}

\newcommand{\R}{{\mathbb R}}
\newcommand{\N}{{\mathbb N}}
\newcommand{\pr}{{\mathop {\rm Pr}}}
\newcommand{\E}{{{\mathbb E}}}

\newcommand{\argmin}{\mathop {\rm argmin}}
\newcommand{\argmax}{\mathop {\rm argmax}}

\newtheorem{defi}{Definition}
\newtheorem{theorem}{Theorem}
\newtheorem{remark}{Remark}
\newtheorem{lemma}[theorem]{Lemma}
\newtheorem{cor}[theorem]{Corollary}

\newcommand{\hc}[1]{\textcolor{blue}{#1}}

\newcommand{\V}[1]{\mathsf{#1}}

\newcommand{\bh}{\V{h}}
\newcommand{\Bh}{\bm{\bh}}
\newcommand{\bx}{\bm{x}}

\newcommand{\bw}{\bm{w}}
\newcommand{\bM}{\V{M}}

\newcommand{\bc}{\bm{c}}

\usepackage{graphicx}

\usepackage{alltt}

\usepackage{newfloat}
\usepackage{listings}
\usepackage{framed}
{\begin{leftbar}\begin{quotation}\noindent}%
{\end{quotation}\end{leftbar}}%
\title{Federated Nearest Neighbor Classification with a Colony of Fruit-Flies:\\
With Supplement}
\author {
    Parikshit Ram\textsuperscript{\rm 1},
    Kaushik Sinha\textsuperscript{\rm 2}
}
\affiliations {
    \textsuperscript{\rm 1}IBM Research AI,
    \textsuperscript{\rm 2}Wichita State University\\
    p.ram@acm.org, kaushik.sinha@wichita.edu

}

\begin{document}
\maketitle

\begin{abstract}
The mathematical formalization of a neurological mechanism in the olfactory circuit of a fruit-fly as a locality sensitive hash (\fh) and bloom filter (\fbf) has been recently proposed and ``reprogrammed'' for various machine learning tasks such as similarity search, outlier detection and text embeddings. We propose  a novel reprogramming of this hash and bloom filter to emulate the canonical nearest neighbor classifier (\nnc) in the challenging Federated Learning (FL) setup where training and test data are spread across parties and no data can leave their respective parties. Specifically, we utilize \fh and \fbf to create the \fnn classifier, and theoretically establish conditions where \fnn matches \nnc. We show how \fnn is trained {\em exactly} in a FL setup with low communication overhead to produce {\fnn}FL, and how it can be differentially private. Empirically, we demonstrate that (i) \fnn matches \nnc accuracy across 70 OpenML datasets, (ii) {\fnn}FL training is highly scalable with low communication overhead, providing up to $8\times$ speedup with $16$ parties.
\end{abstract}
\section{Introduction}
Biological systems (such a neural networks \cite{kavukcuoglu2010learning,krizhevsky2012imagenet}, convolutions \cite{lecun95convolutional}, dropout \cite{hinton2012improving}, attention mechanisms \cite{larochelle2010learning,mnih2014recurrent}) have served as inspiration to modern deep
learning systems, demonstrating amazing empirical performance in areas
of computer vision, natural language programming and reinforcement
learning. Such learning systems are not biologically viable anymore,
but the biological inspirations were critical. This has motivated a
lot of research into identifying other biological systems that can
inspire development of new and powerful learning mechanisms or provide
novel critical insights into the workings of intelligent systems. Such
neurobiological mechanisms have been identified in the olfactory
circuit of the brain in a common fruit-fly, and have been re-used for
common learning problems such as similarity
search~\citep{dasgupta2017neural, ryali2020bio}, outlier
detection~\citep{dasgupta2018neural}, and more recently for word embeddings~\citep{liang2021can} and  centralized classification~\citep{sinha2021fruit}.
More precisely, in the fruit-fly olfactory circuit, an odor activates a small set of Kenyon Cells (KC) which represent a ``tag'' for the
odor. This tag generation process can be viewed as a natural hashing scheme~\citep{dasgupta2017neural}, termed \fh, which generates a high dimensional but very sparse representation (2000 dimensions with 95\%
sparsity). This tag/hash creates a response in a specific mushroom body output neuron (MBON) -- the MBON-$\alpha'3$ -- corresponding to the perceived novelty of the odor. \citet{dasgupta2018neural}
``interpret the KC$\to$MBON-$\alpha'3$ synapses as a Bloom Filter''
that creates a ``memory'' of all the odors encountered by the fruit-fly, and reprogram this {\em Fly Bloom Filter} (\fbf) as a novelty detection mechanism that performs better than other locality sensitive Bloom Filter-based novelty detectors for neural activity and vision datasets. 

We build upon the reprogramming of the KC$\to$MBON-$\alpha'3$ synapses
as the \fbf to create a supervised classification scheme. We show that
this classifier mimics a nearest-neighbor classifier (\nnc). This scheme
possesses several unique desirable properties that allows for
nearest-neighbor classification in the federated learning (FL) setup
with a low communication overhead. In FL setup the complete training
data is distributed across multiple parties and none of the original
data (training or testing) is to be exchanged between the parties. This is possible because
of the unique high-dimensional sparse structure of the \fh.
We consider this an exercise of leveraging ``naturally occurring''
algorithms to solve common learning problems (which these natural
algorithms were not designed for), resulting in schemes with
unique capabilities. Nearest neighbor classification (\nnc) is a fundamental
nonparametric supervised learning scheme, with various theoretical
guarantees and strong empirical capabilities (especially with an appropriate similarity function). FL has gained a lot of
well-deserved interest in the recent years as, on one hand, models
become more data hungry, requiring data to be pooled from various
sources, while on the other hand, ample focus is put on data privacy
and security, restricting the transfer of data. However, the very
nature of \nnc makes it unsuitable for FL
-- for any test point at a single party, obtaining the nearest
neighbors would {\em naively} either require data from all parties to
be collected at the party with the test point, or require the test
point to be sent to all parties to obtain the per-party neighbors;
both these options violate the desiderata of FL.

We leverage the ability of the \fbf to summarize a data distribution
in a bloom filter to develop a classifier where every class is
summarized with its own \fbf, and inference involves selecting the
class whose distribution (represented by its own \fbf) is most similar
to the test point. We theoretically and empirically show that this
classifier, which we name \fnn (Fly Nearest Neighbor) classifier,
approximately agrees with \nnc.
%
%
We then perform \nnc with \fnn on distributed data under the FL setup with low
communication overhead. The key idea is to train a \fnn
separately on each party -- that is, have a ``colony of fruit-flies''
-- and then perform a low communication aggregation at training time without having to exchange any of the original data. This enables low communication federated nearest-neighbor
classification with {\fnn}FL.
One unique capability enabled by this neurobiological mechanism is that {\fnn}FL can perform \nnc without  transferring the test point to other parties in any form.
We make the following contributions\footnote{A preliminary version of this paper was presented at a recent workshop~\citep{ram2021flynn}.}:
\begin{itemize}
\item We present the \fnn classifier utilizing the \fbf and \fh, and theoretically present precise conditions under which \fnn matches the \nnc. 
\item We present an algorithm for training \fnn with distributed data in the FL setting, with low communication overhead and differential privacy, without requiring exchange of the original data.
\item We empirically compare \fnn to \nnc
and other relevant baselines on 70 classification datasets from the OpenML~\citep{van2013openml} data repository.
\item We demonstrate the scaling of the data distributed \fnn training on datasets of varying sizes to highlight the low communication overhead of the proposed scheme.
\end{itemize}
The paper is organized as follows: We
detail the \fnn classifier and analyze its theoretical properties
in \S \ref{sec:flynn}. We present federated \knnc via distributed \fnn in  \S \ref {sec:algo:dist-fbfc}. We empirically evaluate our proposed methods against baselines in \S \ref{sec:emp}, discuss related work in \S \ref{sec:related} and
conclude with a discussion on  limitations and future work in \S \ref{sec:limit}.
%
\section{The \fnn Classifier}
\label{sec:flynn}
In our presentation, we use lowercase letters ($x$) for scalars and functions (with arguments), boldface lowercase letters ($\bx$) for vectors, lowercase SansSerif letter ($\bh$) for Booleans, boldface lowercase SansSerif letter ($\Bh$) for Boolean vectors, and uppercase SansSerif letter ($\bM$) for Boolean matrices. For a vector $\bx$, $\bx[j]$ denotes its $j^{\mathsf{th}}$ index. For any positive integer $k \in \N$, we use $[k]$ to denote the set $\{1, \ldots, k\}$.

We start this section by recalling $k$-nearest neighbor classification (\knnc).
Given a dataset of labeled points  $S = \{(\bx_i, y_i)\}_{i=1}^n \subset
\R^d \times [L]$ from $L$ classes, and a similarity function $s: \R^d \times \R^d \to \R_+$, a test point $\bx \in \R^d$ is labeled by the \knnc based on its $k$-nearest neighbors $S^k(\bx)=\argmax_{R \subset S\colon |R| = k} \sum_{(\bx_i, y_i) \in R} s(\bx, \bx_i)$ as:
%
\begin{equation} \label{eq:knnc}
\hat{y} \gets \argmax_{y \in [L]} \left | \left\{ (\bx_i, y_i) \in S^k(\bx): y_i = y \right \} \right |,
\end{equation}
In the federated version of \knnc, the data is distributed across $\tau$ parties, each with a chunk of the data $S_t, t \in [\tau]$. For a test point $\bx$ at a specific party $t_{\mathsf{in}}$, the classification should be based on the nearest-neighbors of $\bx$ over the pooled data $S_1 \cup S_2 \cdots \cup S_\tau$. A critical desiderata in FL is to be robust to the fact that the per-party data $S_t$ are not obtained from identical distributions for all $t \in [\tau]$ -- the distributions that generate $S_t$ and $S_{t'}$ for $t \not= t'$ can be significantly different. We refer to this as the ``non-IID-ness of the per-party data''.

We leverage the locality sensitive \fh~\citep{dasgupta2017neural} in our proposed scheme, focusing on the binarized version~\citep{dasgupta2018neural}.
For $\bx \in \Real^d$, the \fh $h \colon \Real^d \to \HC^{m}$ is defined as,
\begin{equation} \label{eq:flyhash}
    h(\bx) = \Gamma_\rho (\bM \bx),
\end{equation}
where $\bM \in \HC^{m \times d}$ is the randomized sparse lifting
binary matrix with $s \ll d$ nonzero entries in each row, and
$\Gamma_\rho \colon \Real^{m} \to \HC^{m}$ is the winner-take-all function converting a vector in $\Real^m$ to one in $\HC^m$ by setting
the highest $\rho \ll m$ elements to $1$ and the rest to zero.
\fh projects up or {\em lifts} the data dimensionality ($m \gg d$).

The {\em Fly Bloom Filter} (\fbf) $\bw \in (0,1)^m$ summarizes a dataset and is subsequently used for novelty detection~\citep{dasgupta2018neural} with novelty scores for any point $\bx$ proportional to $\bw^\top h(\bx)$ -- higher values indicate high novelty of $\bx$. To learn $\bw$ from a set $S$, all its elements are initially set to $1$. For an ``inlier'' point $\bx_{\mathsf{in}} \in S$ with \fh $\Bh_{\mathsf{in}}$, $\bw$ is updated by ``decaying'' (with a multiplicative factor) the intensity of the elements in $\bw$ corresponding to the nonzero elements in $\Bh_{\mathsf{in}}$. This ensures that some $\bx \approx \bx_{\mathsf{in}}$ receives a low novelty score $\bw^\top h(\bx)$. For a novel point $\bx_{\mathsf{nv}}$ (with \fh $\Bh_{\mathsf{nv}}$) not similar to any $\bx \in S$, the locality sensitivity of \fh implies that, with high probability, the elements of $\bw$ corresponding to the nonzero elements in $\Bh_{\mathsf{nv}}$ will be close to 1 since their intensities will not have been decayed much, implying a high novelty score $\bw^\top \Bh_{\mathsf{nv}}$.
\subsection{\fnn Algorithm: Training and Inference}
\label{sec:algo:1}
We leverage this mechanism for classification by using the \fbf to summarize each class $l\in [L]$  separately -- the per-class \fbf encodes the local neighborhoods of each class, and the high dimensional sparse nature of \fh (and consequently \fbf) summarizes classes with multi-modal distributions while reducing inter-class \fbf overlap.
\begin{algorithm}[tb]
{\scriptsize
\DontPrintSemicolon
\SetAlgoLined
\caption{\fnn training with training set $S \subset \Real^d \times
  [L]$, lifted dimensionality $m \in \N$, $s \ll d$ nonzeros in each row of the
  lifting matrix $\bM$, $\rho \ll m$ nonzeros in the \fh, decay rate $\gamma \in [0, 1)$, random seed $R$, and inference with test
  point $\bx \in \Real^d$.  }
\label{alg:train-infer}
\SetKwProg{train}{Train{\fnn}$(S, m, \rho, s, \gamma, R)$}
{}
{end}
\SetKwProg{infer}{Infer{\fnn}$(\bx, \bM, \rho, \{\bw_l, l \in [L]\})$}
{}
{end}
\train{}{
  Sample $\bM \in \HC^{m \times d}$ with seed $R$\;
  Initialize $\bw_1,\ldots,\bw_L \gets \mathbf{1}_m \in (0,1)^m$\;
  \For{$(\bx,y)\in S$}{
    $\Bh \gets \Gamma_\rho (\bM \bx)$\;
    $\bw_y[i] \gets \gamma \cdot  \bw_y[i]\ \forall i \in [m]\colon \Bh[i] = 1$\;
  }
  \KwRet{$ (\bM, \{ \bw_l, l \in [L] \}) $} \;
}
\infer{}{
  $\Bh \gets \Gamma_\rho \left( \bM \bx \right)$ \;
  \KwRet{$\argmin_{l \in [L]} \bw_l^\top \Bh$}\;
}
}
\end{algorithm}
\paragraph{\fnn training.}
Given a training set $S \subset \R^d \times [L]$, the learning of the
per-class {\fbf}s $\bw_l \in (0,1)^m, l \in [L]$ is detailed in the
Train\fnn subroutine in Algorithm~\ref{alg:train-infer}. We initialize
the \fh by randomly generating the
$\bM$ (line 2). The per-class \fbf $\bw_l$ are initialized to
$\mathbf{1}_m$ (line 3).
For a training example $(\bx, y) \in S$, we first generate the \fh $\Bh=h(\bx) \in \HC^m$ using equation~\ref{eq:flyhash} (line 5).
Then, the \fbf $\bw_y$ (corresponding to $\bx$'s class $y$) is updated
with the \fh $\Bh$ as follows -- the elements of $\bw_y$ corresponding to
the nonzero bit positions of $\Bh$ are decayed, and the rest of the entries of $\bw_y$ are left as is (line 6); the remaining {\fbf}s $\bw_l, l \not= y \in [L]$ are not updated at all. The decay is achieved by multiplication with a factor of $\gamma \in [0,1)$ -- large  $\gamma$ implies slow decay in the \fbf intensity; a small value of $\gamma$ triggers rapid decay ($\gamma = 0$ makes the {\fbf}s binary).
This whole process ensures that $\bx$ (and points similar to $\bx$) are
considered to be an ``inlier'' with respect to $\bw_y$.
\paragraph{\fnn inference.}
The \fbf $\bw_l$ for class $l \in
[L]$ are learned such that a point $\bx$ with label $l$ appears
 as an inlier with respect to $\bw_l$ (class
$l$); the example $(\bx, y)$ does not affect the other class {\fbf}s $\bw_{l}, l \not= y, l \in [L]$. This implies that a  point $\bx'$ similar to $\bx$ will have a low novelty score $\bw_y^\top h(\bx')$ motivating our inference rule -- for a test point $\bx$, we compute the per-class novelty scores and predict the label as $\hat{y} \gets \arg \min_{l \in [L]} \bw_l^\top h(\bx)$.
This is detailed in the Infer{\fnn} subroutine in Algorithm~\ref{alg:train-infer}.
\subsection{Analysis of \fnn}\label{sec:theory}
We first present the computational complexities of Algorithm~\ref{alg:train-infer} for a specific configuration of its hyper-parameters. All proofs are presented in Appendix~\ref{asec:alg1-proofs}.
%
\begin{lemma}[\fnn training] \label{thm:alg1-train}
Given a training set $S \subset \R^d \times [L]$ with $n$ examples,
the Train\fnn subroutine in Algorithm~\ref{alg:train-infer} with the
lifted \fh dimensionality $m$, number of nonzeros $s$ in each row of
$\bM \in \HC^{m \times d}$, number of nonzeros $\rho$ in
\fh $h(\bx)$ for any $\bx \in \R^d$, and decay rate $\gamma \in [0, 1)$ takes time $O(n m \cdot \max\{s, \log
\rho\})$ and has a memory overhead of $O(m \cdot \max\{s, L\})$.
\end{lemma}
\begin{lemma}[\fnn inference] \label{thm:alg1-infer}
Given a trained \fnn, the inference for $\bx \in \R^d$
with the Infer\fnn subroutine in Algorithm~\ref{alg:train-infer}
takes time $O\left(m \cdot \max\left\{s, \log \rho, \nicefrac{\rho L}{m} \right\}\right)$ with a memory overhead of $O(\max\{m, L\})$.
\end{lemma}
\begin{remark}
For any test point $\bx \in \R^d$ with \fh $h(\bx)$, and a
large number of labels (large $L$), if the $\arg \min_{l \in [L]}
\bw_l^\top h(\bx)$ can be solved via {\bf fast maximum inner product
search}~\citep{koenigstein2012efficient, ram2012maximum} in
time $\beta(L)$ sublinear\footnote{For example, $\beta(L) \sim O(\log L)$ using randomized partition trees~\citep{keivani2017improved, keivani2018improved}.} in $L$, then the overall inference time for
$\bx$ would be given by $O\left(m\cdot \max\left\{s, \log \rho, \nicefrac{\rho \beta(L)}{m} \right\}\right)$ which is sublinear in $L$.
\end{remark}
Next we present learning theoretic properties of \fnn. The novelty score $\bw_l^\top h(\bx)$ of any test point $\bx$ in \fnn corresponds to how ``far'' $\bx$ is from the distribution of class $l$ encoded by $\bw_l$, and using the class with the minimum novelty score to label $\bx$ is equivalent to labeling $\bx$ with the class whose distribution is ``closest'' to $\bx$.
With this intuition, we identify precise conditions where \fnn mimics \knnc.
All proofs are presented in Appendix~\ref{asec:alg1-proofs_learning-theory}.

We present our  analysis for binary classification with $\gamma=0$,
where the \fnn is
trained on  training set $S=\{(\bx_i,y_i)\}_{i=1}^{n}\subset
\R^d \times \{0,1\}$.
Let $S^{0}=\{(x,y)\in S: y=0\}$, $S^{1}=\{(x,y)\in S: y=1\}$
and let $\bw_{0},
\bw_{1}\in\{0,1\}^m$ be the $\fbf$s constructed using $S^{0}$ and $S^{1}$
respectively.
Without loss of generality, for any test point $\bx$,
assume that the majority of its $k$ nearest neighbors from $S$ has class label 1. Thus
\knnc will predict $\bx$'s class label to be 1.
We aim to show that
$\E_{\bM}(\bw_{1}^{\top} h(\bx))< \E_{\bM}(\bw_{0}^{\top}h(\bx))$ (expectation of the random $\bM$ matrix) so that \fnn will predict, in expectation, $\bx$'s label to be 1.
A high probability
statement will then follow using standard concentration bounds. If $\bx$'s nearest neighbor is arbitrarily close to $\bx$ and has label 0 (while the label of the majority of its $k$ nearest neighbors still being 1)  then we would expect  $\bw_{0}^{\top} h(\bx)< \bw_{1}^{\top}h(\bx)$ with high probability, thereby, \fnn will label $\bx$ as $0$. 
To avoid such a situation, we assume a margin $\eta > 0$ between the classes~\citep{gottlieb2014samplecompression} defined as:
\begin{defi}\label{def:margin}
We define the \emph{margin} $\eta$ of the training set $S$ to be $\eta\stackrel{\Delta}{=}\min_{\bx\in S^0,\bx'\in S^1}\|\bx-\bx'\|_{\infty}$.
\end{defi}
If $\lceil\nicefrac{k+1}{2}\rceil$ of $\bx$'s nearest neighbors from $S$ are at a distance at most $\eta/2$ from $\bx$, then all of those $\lceil\nicefrac{k+1}{2}\rceil$ examples must have the same class label to which \knnc agrees. This also ensures that the closest point to $\bx$ from $S$ having opposite label is at least $\eta/2$ distance away. We show next that this is enough to ensure that prediction of \fnn on any test point $\bx$ from a \textbf{\emph{permutation invariant}} distribution agrees with the prediction of \knnc with high probability (the training set can be from any distribution). Note that $P$ is a permutation invariant distribution over $\R^d$ if for any permutation $\sigma$ of $[d]$ and any $\bx\in\R^d$, $P(x_1,\ldots,x_d)=P(x_{\sigma(1)},\ldots,x_{\sigma(d)})$.
\begin{theorem}\label{th:knnc}
Fix $s, \rho, m$ and  $k$. Given a training set $S$ of size $n$ and a test example $\bx\in\R^d$ sampled from a permutation invariant distribution,
let $\bx_*$ be its $\left(\lceil\nicefrac{k+1}{2}\rceil\right)^{th}$  nearest neighbor from $S$ measured using $\ell_{\infty}$ metric. If $\|\bx-\bx_*\|_{\infty}\leq\min\{\nicefrac{\eta}{2},O(\nicefrac{1}{s})\}$ then,  $\hat{y}_{\fnn}=\hat{y}_{\knnc}$ with probability $\geq 1-\left(O(\nicefrac{\rho n}{m})+e^{-O(\rho)}\right)$, where $\hat{y}_{\fnn}$ and $\hat{y}_{\knnc}$ are respectively the predictions of \fnn and  \knnc.
\end{theorem}
\begin{remark}
For any $\delta\in(0,1)$,  the failure probability of the above theorem can be restricted to at most $\delta$ by setting $\rho=\Omega(\log \nicefrac{1}{\delta})$ and $m=\Omega(\nicefrac{n \rho}{\delta})$.
\end{remark}
\begin{remark}
We established conditions under which the predictions of \fnn agrees with
that of \knnc with high probability. \knnc is a non-parametric
classification method with strong theoretical guarantee: as $|S| = n
\to \infty$, the \knnc almost surely approaches the Bayes optimal error rate.
Therefore, by establishing
the connection between \fnn and \knnc, \fnn has the same
statistical guarantee under the conditions of
Theorem~\ref{th:knnc}.
\end{remark}
\begin{remark}
For $k=1$, prediction of \onennc on any $\bx\in\R^d$ agrees with the label of its nearest neighbor $\bx'$ and any point in the training set  having class label different from the label of $\bx'$ is farther away from $\bx$. Here we do not need any dependence on margin $\eta$ and the condition $\|\bx-\bx'\|_{\infty}=O(\nicefrac{1}{s})$ is enough to get a statement similar to Theorem~\ref{th:knnc} that relate  \fnn to  \onennc.
\end{remark}
%
\section{Federated \nnc via Distributed \fnn}\label{sec:algo:dist-fbfc}%
\begin{algorithm}[t]
{\scriptsize  
\DontPrintSemicolon%
\SetAlgoLined%
\caption{\hc{Federated Differentially Private} \fnn training
  \hc{with $\tau$ parties $V_t, t \in [\tau]$ each with training set $S_t$}
  with \hc{DP parameters $\epsilon$ and number of samples $T$}. The boolean {\tt IS\_DP} toggles DP.}
\label{alg:dist-train-infer}
\SetKwProg{train}{Train{\fnn}\hc{FLDP}$(\hc{\{S_t, t \in [\tau]\}}, m, \rho, s, \gamma,{\tt IS\_DP}, \epsilon, T)$}
{}
{end}
\SetKwProg{trainlowcomm}{\hc{LC}Train{\fnn}\hc{FLDP}$(\hc{\{S_t, t \in [\tau]\}}, m, \rho, s, \gamma,\tt{IS\_DP})$}
{}
{end}
\SetKwProg{infer}{\hc{LC}Infer{\fnn}\hc{FLDP}$(\bx, \bM, \rho, \{\bw_l^t, l \in [L], t \in [\tau]\},\tt{IS\_DP})$}
{}
{end}
\SetKwProg{DP}{DP$(\{\bw_l,l\in {L}\},\epsilon,T)$}
{}
{end}
\train{}{
  \hc{Generate random seed $R$ \& broadcast to all $V_t, t \in [\tau]$}\;
  
  \For{\hc{each party $V_t, t \in [\tau]$}}{
    $(\bM, \{\bw_l^t, l \in [L]\}) \gets$ Train\fnn$(\hc{S_t}, m, \rho, s, \gamma, \hc{R})$\;
     \If{\hc{${\tt IS\_DP}$}}{\{\hc{$\bw^t_l,l\in [L]\} \gets$ DP$((\{\bw_l^t,l\in {L}\}, \epsilon/\tau, T))$}}
  }
  \tcp*[l]{All-reduction over all $\tau$ parties}
  \hc{
    $\hat{\bw}_l[i] \gets \gamma^{\sum_{t \in [\tau]}  \log_{\gamma} \bw_l^t[i]} \forall i \in [m], \forall l \in [L]$
  }\;
  \KwRet{
    $ (\bM, \{ \hat{\bw}_l, l \in [L] \}) $
    \hc{on each party $V_t, t \in [\tau]$}
  }
}
\DP{}{
  $\bw \gets [(\bw_1)^{\top},\ldots,(\bw_L)^{\top}]^{\top}$\;
  $\bc[i] \gets \log_{\gamma} \bw[i]\ \forall i \in [m\times L]$\;
  $\mathcal{R} \gets \{\}$\;
  \For{$j\in [T]$}{
    Sample $i_j\in[m\times L]$ with probability $\propto \exp\left(\epsilon\bc[i_j] / 4T \right)$\;
    $\bc[i_j] \gets \max\{\bc[i_j]+\eta,0\}$, where $\eta \sim \text{\sf Laplace}(2T/\epsilon)$\;
    $\mathcal{R} \gets \mathcal{R}\cup \{i\}$\;
  }
  $\bc[i] \gets  0,  \forall i\in [m\times L]\setminus \mathcal{R}$\;
  $\tilde{\bw}[i] \gets \gamma^{\bc[i]}, \forall i\in [m\times L]$\;
  $ [(\tilde{\bw}_1)^{\top},\ldots,(\tilde{\bw}_L)^{\top}]^{\top}\gets \tilde{\bw}$\;
  \KwRet{$ (\{ \tilde{\bw}_l, l \in [L] \}) $} \;
}
}
\end{algorithm}%
For federated learning where the data
$S$ is spread across $\tau$ parties with each
party $V_t, t \in [\tau]$ having its own chunk $S_t$, we present a distributed
{\fnn}FL learning scheme in Algorithm~\ref{alg:dist-train-infer}, highlighting the differences from the original \fnn learning in \hc{Blue} text. The boolean {\tt IS\_DP} toggles the differential privacy (DP) of the training. This scheme ensures {\em inter-party privacy}, protecting against leakage even with colluding parties.
%
The proofs for the analyses in this section are presented in Appendix~\ref{asec:alg2-proofs}].

In Train{\fnn}\hc{FLDP}, all the parties
$V_t, t \in [\tau]$ have the complete \fnn model at the conclusion of the training, and are {\em able to
perform no-communication inference on any new test point $\bx$
independent of the other parties} using the Infer\fnn subroutine in Algorithm~\ref{alg:train-infer}. The learning commences by generating and broadcasting a
random seed $R$ to all parties $V_t, t \in [\tau]$ (line 2); we assume
that all parties already have knowledge of the total number of labels
$L$. Then each party $V_t$ independently invokes Train{\fnn} (Algorithm~\ref{alg:train-infer}) on its chunk $S_t$ and obtains the per-class private or non-private \fbf $\{\bw_l^t, l \in [L]\}$ depending on the status of the boolean variable $\tt{IS\_DP}$ and the invocation of the \hc{DP} subroutine (lines 3-8). Finally, a specialized {\em all-reduce} aggregates all the per-class
{\fbf}s $\{\bw_l^t, l \in [L]\}$ across all parties $t \in [\tau]$
to obtain the final {\fbf}s $\hat{\bw}_l, l \in [L]$ on all parties (line 9). In the \hc{DP} module, the input \fbf $\{\bw_l, l \in [L]\}$ are concatenated and the element-wise $\log$ values (counts) are stored in a vector $\bc$ (lines 13-14). Then the largest $T$ indices of $c$ are selected iteratively using an exponential mechanism and Laplace noise are added to these selected entries (lines 16-20). The remaining $(m\times L) -T$ entries of $\bc$ are set to zero (line 21), all the entries of $\bc$ are exponentiated and the differentially private \fbf $\{\tilde{\bw}_l^t, l \in [L]\}$ are returned (lines 22-24).
The following claim establishes exact parity between the non-DP federated and original training of \fnn:
\begin{theorem}[Non-private Federated training parity]\label{thm:parity-1}
Given training sets $S_t \subset \R^d \times [L]$ on each party $V_t,
t \in [\tau]$, and a {\fnn} configured as in Lemma~\ref{thm:alg1-train}, if the boolean variable $\tt{IS\_DP}$ is False, then the
  per-party final \fnn $\{\hat{\bw}_l, l \in [L]\}$
  (Alg.~\ref{alg:dist-train-infer}, line 9) output by
  Train{\fnn}\hc{FLDP}~$(\{S_t, t \in [\tau]\}, m, s,
  \rho, \gamma, {\tt IS\_DP}, \epsilon, T)$ with random seed $R$ in Algorithm~\ref{alg:dist-train-infer} is equal
  to the \fnn $\{\bw_l, l \in [L]\}$ (Alg.~\ref{alg:train-infer}, line
  8) output by Train\fnn~$(S, m, s, \rho, c, R)$ subroutine in
  Algorithm~\ref{alg:train-infer} with the pooled training set $S = \cup_{t \in
    [\tau]} S_t$.
\end{theorem}
This implies that the {\fnn}FL training  (i) {\em does not incur any approximation}, {\bf and}
(ii) {\em does not require any original training data to leave their respective parties},
and these aggregated per-class {\fbf}s are now available
on every party $V_t$ and used to (iii) {\em perform inference on test points
on each party with no communication to other parties} using the
Infer\fnn subroutine in Algorithm~\ref{alg:train-infer}. The unique
capabilities are enabled by the learning dynamics of the
\fbf in \fnn.
\begin{remark}[Agnostic to non-IID-ness of per-party data] \label{rem:non-iid-agnostic}
Theorem~\ref{thm:parity-1} implies that the proposed {\fnn}FL is completely agnostic to the non-IID-ness of the data across parties.
The proposed scheme approximates the ideal \knnc (which has unrestricted access to data from all the parties) {\bf regardless of the non-IID-ness  of the per-party data}. 
\end{remark}
The computational complexities of {\fnn}FL training are as follows:
\begin{lemma}[{\fnn}FL training] \label{thm:alg3-train}
Given the setup in Theorem~\ref{thm:parity-1} 
with $|S_t| = n_t$, Train{\fnn}\hc{FLDP} (Alg.~\ref{alg:dist-train-infer}) takes $O(m \cdot \max\{s, L\})$ memory, with (a) $O\left(n_t  m
\cdot \max\left\{ s, \log \rho, \nicefrac{L}{n_t} \log \tau \right\}\right)$ time per-party and $O(mL \tau)$ communication with DP disabled ({\tt IS\_DP}=false), and (b) $O\left(n_t  m
\cdot \max\left\{ s, \log \rho, \nicefrac{LT}{n_t}\right\} + T \log \tau \right)$ time per-party and $O(T \tau)$ communication with DP enabled.
\end{lemma}
The following result establishes the DP property of Train{\fnn}FLDP, which prevents leakage between parties during the training procedure. The proof leverages the exponential mechanism.
\begin{theorem}\label{th:dp_proof}
With the DP module enable ({\tt IS\_DP}=true), Train{\fnn}\hc{FLDP} is
$(\epsilon,0)$ differentially private.
\end{theorem}
\paragraph{Communication setup.}
The  Train{\fnn}FLDP algorithm is presented
here in a peer-to-peer communication setup. However,
it will easily transfer to a centralized
setup with a ``global aggregator'' that all parties
communicate to. In that case, for {\fnn}FL training, the aggregator (i) generates and broadcasts the
seed, and (ii) gathers \& computes $\{\hat{\bw}_l, l \in [L]\}$, and (iii) broadcasts them to all parties. 
\paragraph{Effect of timed-out parties (``Stragglers'') in FL.}
The ability to be robust to stragglers is of critical importance in FL. Stragglers play an important role in iterative algorithms with multiple rounds of communication.
In Train{\fnn}FLDP, there is only {\em a single round of communication} in the training scheme (Algorithm~\ref{alg:dist-train-infer}, line 9), we do not anticipate there to be any stragglers. For inference, {\em all computations are local to each party}, and hence, there is no notion of stragglers.
We leave further study of stragglers for future work.
%
%
\section{Empirical Evaluation}\label{sec:emp}
In this section, we evaluate the empirical performance of
\fnn. First, we compare \fnn to \nnc to validate its ability to approximate \nnc. Then, we demonstrate the scaling of {\fnn}FL training on data distributed among multiple parties. Finally, we present the privacy-performance tradeoff of {\fnn}FL. Various details and additional experiments are presented in Appendix~\ref{asec:emp-eval}. The implementation details and compute resources used are described in
Appendix~\ref{asec:emp-eval:details} and relevant code is available at \url{https://github.com/rithram/flynn}.
\paragraph{Datasets.}
For the evaluation of \fnn, we consider three groups of datasets:
\begin{itemize}
\item We consider binary and continuous {\bf synthetic data} of varying sizes,
designed to favor local classifiers like \nnc~\citep{guyon2003design}. See Appendix~\ref{asec:emp-eval:syndata} for further details.
\item We consider 70 classification datasets from OpenML~\citep{van2013openml}
to evaluate the performance of \fnn on real datasets, {\bf thoroughly} comparing \fnn to \nnc. See Appendix~\ref{asec:emp-eval:openml} for details.
\item We consider high dimensional vision datasets MNIST~\citep{lecun1995MNIST}, Fashion-MNIST~\citep{xiao2017fashion} and CIFAR~\citep{krizhevsky2009learning} from the Tensorflow package~\citep{abadi2016tensorflow}
for evaluating the scaling of {\fnn}FL training when the data is distributed between multiple parties. See [Appendix~\ref{asec:emp-eval:scaling} for details.
\end{itemize}

\paragraph{Baselines and ablation.}
We compare our proposed \fnn to two baselines:
\begin{itemize}
\item {\bf \knnc:} This is the primary baseline. We tune over the
  neighborhood size $k \in [1,64]$. We also specifically consider
  \onennc ($k = 1)$.
\item {\bf \sbfc:} To ablate the effect of the high level of sparsity in \fh, we utilize the binary \sh~\citep{charikar2002similarity} based locality sensitive bloom filter for each class in place of \fbf to get \sh Bloom Filter classifier (\sbfc). See Appendix~\ref{asec:emp-eval:sbfc} for further details.
\end{itemize}
\paragraph{\fnn hyper-parameter search.}
For a dataset with $d$ dimensions, we tune across $60$ \fnn
hyper-parameter settings in the following ranges: $m \in [2d, 2048d]$, $s \in [2, \lfloor 0.5d \rfloor]$, $\rho \in [8, 256]$, and $\gamma \in [0, 0.8]$. We use this hyper-parameter space for all experiments, except for the vision sets, where we use $m \in [2d, 1024d]$.
We present various experiments and detailed discussions on the hyper-parameter dependence in Appendix~\ref{asec:emp-eval:hpdep}.
To summarize the dependence, (i)~increasing $m$ improves \fnn performance and can be selected to be as large as computationally feasible, (ii)~when $d$ is large enough ($\geq 20$), the \fnn performance is somewhat agnostic to the choice of $s$ and any small value ($s \sim 0.05d$) suffices, (iii)~increasing $\rho$ improves \fnn performance up to a point after which it can hurt performance unless $m$ is increased as well since it reduces the sparsity of \fh, (iv)~increasing $\gamma$ from $0$ to $>0$ significantly improves \fnn performance, but otherwise the performance is quite robust to its precise choice.
%
\paragraph{Evaluation metric to compare across datasets.}
To obtain statistical significance and error bars for performance across different datasets, we compute the ``normalized accuracy'' for a method on a dataset as $(1 - \nicefrac{a}{a_k})$ where $a_k$ is the best tuned 10-fold cross-validated accuracy of \knnc on this dataset
and $a$ is the best tuned 10-fold cross-validated accuracy obtained by the method on this dataset.
Thus \knnc has a normalized accuracy of $0$ for all datasets; negative values denote improvement over \knnc. This ``normalization'' allows comparison of the aggregate performance of different methods across different datasets with distinct best achievable accuracies.
\paragraph{Synthetic data.}
\begin{table}[tb]
\caption{Comparison of \fnn with \nnc and
  \sbfc on synthetic data. We report normalized accuracy aggregated over 30 random synthetic datasets. Normalized accuracy for \knnc is zero hence elided from the results.
}
\label{tab:syn-data}
\begin{center}
{\scriptsize
\begin{tabular}{cccccc}
\toprule
$\mathcal X$      & $n$  & $d$  & \onennc  & \sbfc & \fnn \\
\midrule
$\HC^d$ & $10^3$ & $50$  & $0.11 \pm 0.05$ & $0.18 \pm 0.04$ & $-0.05 \pm 0.02$ \\
$\HC^d$ & $10^4$ & $50$  & $0.18 \pm 0.02$ & $0.51 \pm 0.02$ & $-0.03 \pm 0.01$ \\
$\HC^d$ & $10^3$ & $100$ & $0.07 \pm 0.03$ & $0.58 \pm 0.03$ & $-0.04 \pm 0.02$ \\
$\R^d$  & $10^3$ & $50$  & $0.09 \pm 0.03$ & $0.68 \pm 0.01$ & $-0.07 \pm 0.02$ \\
$\R^d$  & $10^4$ & $50$  & $0.11 \pm 0.00$ & $0.78 \pm 0.00$ & $0.07 \pm 0.02$  \\
$\R^d$  & $10^3$ & $100$ & $0.11 \pm 0.03$ & $0.66 \pm 0.02$ & $-0.05 \pm 0.03$ \\
\bottomrule
\end{tabular}
}
\end{center}
\end{table}
We first consider synthetic data designed for strong \knnc
performance. We generate data for 5 classes with 3 clusters per class,
and points in the same cluster belong to the same class implying that
a neighborhood based classifier will perform well. However, the
classes are not linearly separable. We select such a set
to demonstrate that the proposed \fnn is able to encode multiple
separate modes of a class within a single \fbf while providing
enough separation between the per-class {\fbf}s for high predictive 
performance.
We consider binary synthetic data in $\HC^d$ and synthetic data in general $\Real^d$. We consider $d = 50, 100$ and $n = 10^3, 10^4$. For each configuration, we create 30 datasets. The performances of all baselines are presented in Table~\ref{tab:syn-data}.
The results indicate that \fnn is able to match \knnc performance
significantly better than all other baselines, including \onennc, by being closest to zero (\fnn appears to improve upon \nnc but the improvements are not significant overall). \fnn significantly outperforms \sbfc, highlighting the need for {\em sparse high dimensional hashes} to summarize multi-modal distributions while avoiding overlap between per-class {\fbf}s. The small standard errors indicate the stability of the relative performances across different datasets.
\paragraph{OpenML data.}
\begin{figure}[t]
  \centering
    \includegraphics[width=0.4\textwidth]{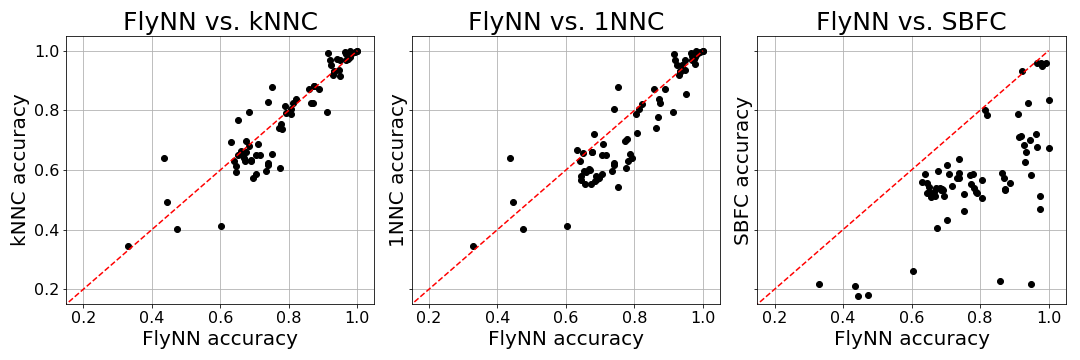}
  \caption{Performance of baselines relative to \fnn on {\em OpenML} datasets. The scatter plots compare the best tuned \fnn accuracy against that of \knnc, \onennc and \sbfc with a point for each dataset, and the red dashed diagonal marking match to \fnn accuracy.}
  \label{fig:flynn-base}
\end{figure}
\begin{table}[tb]
\caption{Evaluating \fnn on OpenML datasets with (i)~Fraction of datasets \fnn outperforms baselines, (ii)~Number of datasets
  \fnn has wins(W)/ties(T)/losses(L) over baselines, (iii)~Median improvement in normalized accuracy by \fnn over baseline, (iv)~$p$-values for the paired 2-sided t-test (TT), (v)~$p$-values for the 2-sided Wilcoxon signed rank test
  (WSRT).}
\label{tab:openml-ind-comp}
\begin{center}
{\scriptsize%
\begin{tabular}{lccccc}
\toprule
\textsc{Method} & (i) {\sc Frac.} & (ii) {\sc W/T/L} & (iii) {\sc Imp.} & (iv) {\sc TT} & (v) {\sc WSRT}  \\
\midrule
\knnc           & 0.55            & 39/2/30    & 0.35\% & 5.30E-2 & 7.63E-2  \\
\onennc         & 0.66            & 47/2/22    & 2.36\% & 1.55E-5 & 2.81E-5  \\
\sbfc           & 0.99            & 70/0/1     & 25.4\% & $<$1E-8 & $<$1E-8  \\
\bottomrule
\end{tabular}
}%
\end{center}
\end{table}
We consider 70 classification (binary and multi-class) datasets from OpenML with $d$ numerical columns and $n$ samples; $d \in [20, 1000], n \in [1000, 20000]$. Unlike the synthetic sets, these datasets do not guarantee strong \knnc performance. The results are presented in Figure~\ref{fig:flynn-base}. 
In Table~\ref{tab:openml-ind-comp}, the normalized accuracy of all baselines are compared to \fnn with paired two-sided $t$-tests (TT) and two-sided Wilcoxon signed rank test (WSRT).
In Figure~\ref{fig:flynn-base}, we can see on the left figure (\knnc vs \fnn) that most points are near the diagonal (implying \knnc and \fnn parity) with some under (better \fnn accuracy) and some over (worse \fnn accuracy). With \onennc in the center plot of Figure~\ref{fig:flynn-base}, we see that, in most cases, \onennc either matches \fnn  or does worse (being under the diagonal) since \knnc subsumes \onennc. But the right plot for \sbfc in Figure~\ref{fig:flynn-base} indicates that \sbfc is quite unable to match \fnn (and hence \knnc). 
We quantify these behaviours in Table~\ref{tab:openml-ind-comp}. \fnn performs comparably to {\knnc} (median improvement of only 0.35\%) with $p$-values of $0.0536$ (TT) and 0.0763 (WSRT), while improving the normalized accuracy over \onennc by a median of around 2.36\% across all 70 sets ($p$-values $\sim 10^{-5}$). These results demonstrate that {\em the proposed \fnn has comparable performance to properly tuned \knnc and this behaviour is verified with a large number of datasets}. \fnn significantly outperforms \sbfc ($>25\%$ median improvement, $p$-values  $<10^{-8}$), again highlighting the value of high sparsity in the \fh on real datasets.

Methods learned through gradient-descent (such as linear models or neural networks) have been widely studied in the FL setting. However, it is hard to compare nearest-neighbor methods against gradient-descent-based methods with proper parity. We present one comparison on these OpenML datasets in Appendix~\ref{asec:emp-eval:gd-based}.
\paragraph{Scaling.}
\begin{figure}[t]
  \centering
  \includegraphics[width=0.27\textwidth]{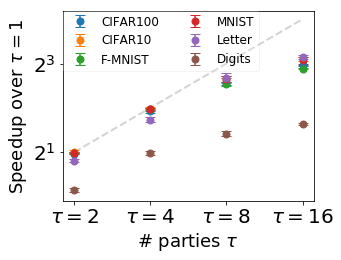}
  \caption{Scaling of {\fnn}FL training with $\tau$ parties
    over single-party training. 
    The \textcolor{darkgray}{gray} line marks linear scaling.}
  \label{fig:scaling}
\end{figure}
We evaluate the scaling of the {\fnn}FL training --
Algorithm~\ref{alg:dist-train-infer},
Train{\fnn}\hc{FLDP} -- with the number of
parties $\tau$. For fixed hyper-parameters, we average runtimes (and speedups) over 10 repetitions for each of the 6 datasets (see Appendix~\ref{asec:emp-eval:scaling}) and present the results in Figure~\ref{fig:scaling}.  The results indicate
that Train{\fnn}\hc{FLDP} scales very well for up to 8
parties for the larger datasets, and shows up to $8 \times$ speed up with $16$ parties. There is significant gain (up to $2 \times$) even for the tiny \textsc{Digits} dataset (with $<2000$ total rows), demonstrating the scalability of the \fnn training with very low communication overhead.  
\paragraph{Differential privacy.}
\begin{figure}[t]
  \centering
  \begin{subfigure}{0.22\textwidth}
    \centering
    \includegraphics[width=\textwidth]{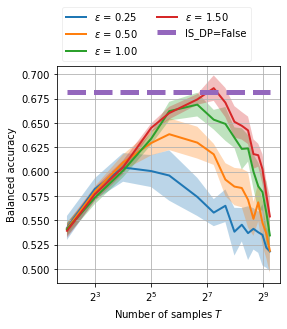}
    \caption{Effect of $T$ for different $\epsilon$.}
    \label{fig:dp-eps-T}
  \end{subfigure}
  ~
  \begin{subfigure}{0.225\textwidth}
    \centering
    \includegraphics[width=\textwidth]{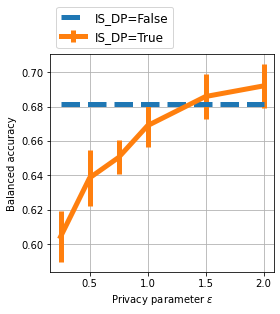}
    \caption{Dependence on $\epsilon$.}
    \label{fig:dp-eps-v-acc}
  \end{subfigure}
  \caption{The effect of the DP parameters $T$ and $\epsilon$ on the performance of {\fnn}FL (higher is better).}
  \label{fig:dp-one-hp}
\end{figure}
To study the privacy-performance tradeoff of \fnn, we again consider the previously described synthetic data in $\R^d$. For a fixed setting \fnn hyper-parameters (namely $m, s, \rho, \gamma$) and 2-party federated training ($\tau=2$), we study the effect of the privacy parameters on {\fnn}FL performance in Figure~\ref{fig:dp-one-hp}. In Figure~\ref{fig:dp-eps-T}, we see the effect of varying the number of sampled entries $T$. We observe the intuitive behaviour where, for a fixed privacy level $\epsilon$, increasing $T$ initially improves performance, but eventually hurts because of the high noise level obfuscating the \fbf entries too drastically. In Figure~\ref{fig:dp-eps-v-acc}, we report the effect of varying $\epsilon$. For each $\epsilon$, we select the $T$ corresponding to the best performance (based on Figure~\ref{fig:dp-eps-T}). This shows the expected trend of increasing performance with increasing $\epsilon$, where the DP \fnn can match the non-DP \fnn ({\tt IS\_DP}=false) with $\epsilon$ close to 1.
We present further experimental details and results for different dataset sizes and different \fnn hyper-parameter configurations in Appendix~\ref{asec:emp-eval:dp}.
\section{Related work}\label{sec:related}
The $k$ nearest neighbor classification method (\knnc)  is a conceptually simple, non-parametric and
popular classification method which defers its computational burden to the prediction stage. The consistency properties of \knnc are well studied \cite{fix1951_consistency,cover1967_nn,devroye1994_consistency,chaudhuri2014_rate}. 
Traditional \knnc assumes training data is stored centrally in a single machine, but such  central processing and storing assumptions become unrealistic
in the big data era.
An effective way to overcome this issue is to 
distribute the data
across multiple machines and  use
specific
distributed computing environment such as Hadoop or Spark
with MapReduce paradigm  \cite{anchalia2014mapreduce_knn, mallio2015_mapreduce_knn,gonzalez-lopez2018_disttrubuted_knn}.
\citet{zhang2020_disttrubuted_knn} proposed a \knnc algorithm based on the concept of distributed storage and computing for processing large datasets in cyber-physical systems where $k$-nearest neighbor search is performed locally using a kd-tree. \citet{qiao2020_largescale_knn} analyzed a distributed \knnc in which data are divided into multiple smaller subsamples, \knnc predictions are made locally in each subsamples and these local predictions are combined via majority voting to make the final prediction.
Securely computing \knnc is another closely related field when data is stored in different local devices. Majority of the frameworks that ensure privacy for \knnc often use some sort of secure multi-party computation (SMC) protocols \cite{zhan2005_privacy,xiong2006_private,qi2008_privacy,schoppmann2020_secure,shaul2020_secure,chen2020_secure_approximate}.

The federated learning framework involves training statistical models over remote devices 
while keeping data localized. Such a framework has recently received significant attention with the growth of the storage and computational capabilities of the remote devices within distributed networks especially because learning in such setting differs significantly from tradition distributed environment requiring advances in areas such as security and privacy, large-scale machine learning and distributed optimization.
Excellent survey and research questions on this new field can be found in \cite{li2020_FL_survey,kairouz2021_FL_survey}. 
In federated learning,
the parameters of a global model is learned in rounds, where in each round a central server sends the current state of the global algorithm (model parameters) to all the parties, each party makes local updates and sends the updates back to the central server \cite{mcmahan2017_fedaverage}. 

Current distributed \knnc schemes do not directly translate to the federated learning setting since the test point needs to be transmitted to all parties.
In most secure \knnc settings considered in the literature, the goal is to keep the training data secure from the party making the test query~\cite{qi2008_privacy,shaul2018secure,wu2019_privacy_knn} and it is not clear how those approaches extend to the multi-party federated setting where the per-party data (train or test) should remain localized.
Of particular relevance is \citet{schoppmann2020_secure}\footnote{A previous version of \citet{schoppmann2020_secure} was titled ``Private Nearest Neighbors Classification in Federated Databases'' (\url{https://eprint.iacr.org/eprint-bin/versions.pl?entry=2018/289}) but has since been changed as the focus of the paper was shifted.} which proposed a scheme to compute a secure inner-product between any test point and all training points (distributed across parties) and then perform a secure top-$k$ protocol
to perform \knnc.
This procedure explicitly computes the neighbors for a test point, which involves $n$ secure similarity computations for each test point (on top of the secure top-$k$ protocol).
Both these steps require significant communication at inference time. In contrast, our proposed scheme do not require any explicit computation of the nearest-neighbors and hence requires no top-$k$ selection (secure or otherwise). In fact the inference requires {\em no communication} and the training can be made DP.
Moreover, this paper focuses on document classification and leverages the significantly sparse feature representations of training examples. The high sparsity allows the use of {\em correlated permutations} to compute inner-products. However, the critical use of correlated permutations of the non-sparse indices augmented with padding does not translate to general dense data -- in the absence of sparsity, the required correlated permutations would be very large and require multiple rounds of computation for a single similarity computation.
Hence it is not clear if these techniques translate to the general \knnc.
\section{Limitations \& Future Work}\label{sec:limit}
A high-level limitation of our work is that we are not presenting a state-of-the-art result, but rather demonstrating how naturally occurring algorithms (\fh and \fbf) can expand the scope of a canonical ML scheme (\nnc) to more learning environments (FL). Our motivations are not to achieve state-of-the-art, but rather to explore and understand the novel unique capabilities of this neurobiological motif.

Another limitation is that the current theoretical connection between \fnn and \knnc requires assumptions on the class margins and on the distribution of the data (the test point is from a permutation invariant distribution). This limits the scope of the theoretical result though we try to verify the theory with a large number of synthetic and real datasets. We plan to remove such assumptions in our future work.
\paragraph{Acknowledgement.} 
KS gratefully acknowledges funding from the ``NSF AI Institute for Foundations of Machine Learning (IFML)'' (FAIN: 2019844).

{
\bibliography{references}
}
\clearpage
\onecolumn
\appendix
\section{Analyses of Algorithm~\ref{alg:train-infer} complexities} \label{asec:alg1-proofs}

\subsection{Proof for Lemma~\ref{thm:alg1-train}}

\begin{proof}
We can summarize the complexities for the different operations in
Train\fnn (Algorithm~\ref{alg:train-infer}) as follows:
\begin{itemize}
\item Line 2 takes $O(ms)$ time and memory to generate the random
  binary lifting matrix $\bM$.
\item Line 3 takes $O(mL)$ time and memory to initialize the per-class
  {\fbf}s $\bw_l, l \in [L]$.
\item Each \fh in line 5 takes $O(m(s + \log \rho))$ time and $O(m)$
  memory.
\item \fbf $\bw_y$ update in line 6 takes $O(\rho)$ time since
  $\gamma$ is multiplied to only $\rho$ entries in $\bw_y$ since $\Bh$
  has only $\rho$ non-zero entries.
\item Hence the loop 4-7 takes time $O(n (ms + m \log \rho + \rho))$
  and maximum $O(m)$ additional memory.
\item Given $\rho \ll m$ and $L \ll n$, the total runtime is given by
  $O(n \cdot m \cdot \max\{s, \log \rho\})$ time and $O(m \cdot
  \max\{s, L\})$ memory.
\end{itemize}
This proves the statement of the claim.
\end{proof}

\subsection{Proof for Lemma~\ref{thm:alg1-infer}}

\begin{proof}
We can summarize the complexities for the different operations in
Infer\fnn (Algorithm~\ref{alg:train-infer}) as follows:
\begin{itemize}
\item The \fh operation in line 11 takes time $O(m(s + \log \rho))$
  and $O(m)$ memory.
\item The operation in line 12 takes time $O(L \rho)$ since each
  $\bw_l^\top \Bh$ takes time $O(\rho)$ since $\Bh$ only has $\rho$
  nonzero entries and $O(L)$ additional memory.
\item This leads to an overall runtime of $O(m \cdot \max\{s, \log
  \rho, (\rho L / m)\})$ and memory overhead of $O(\max\{m, L\})$.
\end{itemize}
This proves the statement of the claim.
\end{proof}

\section{Analysis of Algorithm~\ref{alg:train-infer} learning theoretic properties} \label{asec:alg1-proofs_learning-theory}

\subsection{Preliminaries \& notations}

We denote a single row of a lifting matrix $\bM$ by
$\theta\in\{0,1\}^d$ drawn i.i.d. from $Q$, the uniform distribution
over all vectors in $\{0,1\}^d$ with exactly $s$ ones, satisfying
$s\ll d$.
We use an alternate formulation of the winner-take-all strategy as
suggested in~\citet{dasgupta2018neural}, where for any $\bx\in\R^d,
\tau_{\bx}$ is a threshold that sets largest $\rho$ entries of $\bM\bx$ to one
(and the rest to zero) in expectation.
Specifically, for a given $\bx\in\R^d$ and for any fraction $0<f<1$, we
define $\tau_{\bx}(f)$ to be the top $f$-fractile value of the
distribution $\theta^{\top}\bx$, where $\theta\sim Q$:
\begin{equation}
    \tau_{\bx}(f)=\sup\{v: \pr_{\theta\sim Q}(\theta^{\top}\bx\geq v)\geq f\}
\end{equation}
We note that for any $0<f<1$, $\pr_{\theta\sim Q}(\theta^{\top}\bx\geq
\tau_{\bx}(f))\approx f$, where the approximation arises from possible
discretization issues. For convenience, henceforth we will assume that
this is an equality:
\begin{equation}
    \pr_{\theta\sim Q}(\theta^{\top}\bx\geq \tau_{\bx}(f))= f
\end{equation}
For any two $\bx,\bx'\in\R^d$, we define $q(\bx,\bx')=\pr_{\theta\sim Q}
\left(\theta^{\top}\bx'\geq \tau_{\bx'}\left( \rho / m \right) ~|~
\theta^{\top}\bx \geq\tau_{\bx}\left( \rho / m \right) \right)$.
This can be interpreted as follows -- with $\Bh=h(\bx), \Bh'=h(\bx')$ as the
$\fh$es of $\bx$ and $\bx'$, respectively, $q(\bx,\bx')$ is the probability
that $\Bh'[j]=1$ given that $\Bh[j]=1$, for any specific $j \in [m]$.

We begin by analyzing the binary classification performance of \fnn
trained on a training set $S=\{(\bx_i,y_i)\}_{i=1}^{n_0+n_1}\subset
\R \times \{0,1\}$, where $S=S^{1}\cup S^{0}$, $S^{0}$ is a
subset of $S$ having label 0, and $S^{1}$ is a subset of $S$ having
label 1, satisfying $|S^{0}|=n_0$, $|S^{1}|=n_1$ and
$n=n_0+n_1$. For appropriate choice of $m$, let $\bw_{0},
\bw_{1}\in\{0,1\}^m$ be the $\fbf$s constructed using $S^{0}$ and $S_1$
respectively.

\subsection{Connection to \knnc}
We first present the following lemmas which will be required to prove Theorem~\ref{th:knnc}.
\begin{lemma}[Expected novelty response \cite{dasgupta2018neural}]\label{alem:expected_novelty}
Suppose that inputs $\bx_1,\ldots,\bx_n\in\R^d$ are first presented with $\bx_i\rightarrow \bm{y}_i\rightarrow\bh_i$, where $\bm{y}_i=\bM\bx_i, \bh_i=h(\bx_i)$, and $\bw$ is the \fbf constructed using $\bx_1,\ldots,\bx_n$. Then a subsequent input $\bx$ is presented with $\bx\rightarrow\bm{y}\rightarrow\bh$, where $\bm{y}=\bM\bx, \bh=h(\bx)$.

\noindent (a) The $m$ random vectors $(\bm{y}_1[j],\ldots,\bm{y}_n[j],\bh_1[j],\ldots,\bh_n[j],\bm{y}[j],\bh[j],\bw[j]), 1\leq j\leq m$, (over the random choice choice of $\bM$) are independent and identically distributed.

\noindent (b) The novelty response to $\bx$ has expected value
$$\mu=\E(\bw^{\top}h(\bx)/\rho)=\pr_{\theta\sim Q}\left(\theta^{\top}\bx_1<\tau_{\bx_1},\ldots,\theta^{\top}\bx_{n}<\tau_{\bx_n} ~|~ \theta^{\top}\bx\geq \tau_{\bx}\right)$$
\end{lemma}

\begin{lemma}[Bounds on expected novelty response  \cite{dasgupta2018neural}]\label{alem:bound_novelty}
The expected value $\mu$ from Lemma~\ref{alem:expected_novelty} can be bounded as follows:

\noindent (a) Lower bound:  $\mu \geq 1-\sum_{i=1}^n q(\bx, \bx_i)$.

\noindent (b) Upper bound: for any $1\leq l\leq n$, $~~\mu\leq 1- q(\bx, \bx_l)$.
\end{lemma}

\begin{lemma}[\cite{dasgupta2018neural}]\label{alem:q_lowerbound}
Pick any $\bx, \bx'\in\R^d$. Suppose that for all $i\in [d]$, $\bx'[i]\geq \bx[i]-\Delta/s$, where $\Delta=\frac{1}{2}\left(\tau_{\bx}(\rho/2m)-\tau_{\bx}(\rho/m)\right)$. then $q(\bx, \bx')\geq 1/2$.
\end{lemma}

\begin{cor}[\cite{dasgupta2018neural}]\label{acor:expected_q}
Fix any $\bx'\in\R^d$ and pick $\bx$ from any permutation invariant distribution over $\R^d$. then the expected value of $q(\bx, \bx')$, over the choice of $\bx$ is $\rho/m$.
\end{cor}

\begin{lemma}\label{alem:expectation}
Fix any $\bx\in\R^d$ and let  $h(\bx)\in\{0,1\}^m$ be its \fh using equation~\ref{eq:flyhash}. For any integer $k$, let $\bx_{*}^{i}$ be the $\left(\lceil\frac{k+1}{2}\rceil\right)^{th}$ nearest neighbor of $\bx$ in $S^i$ measured using $\ell_{\infty}$ metric.
Let $A_{S^1}=\{\theta: \cap_{(\bx',y')\in S^{1}} ~ \theta^{\top} \bx'<\tau_{\bx'}(\rho/m)\}$ and $A_{S^0}=\{\theta: \cap_{(\bx',y')\in S^{0}} ~ \theta^{\top} \bx'<\tau_{\bx'}(\rho/m)\}$. Then the following holds, where the expectation is taken over the random choice of projection matrix $\bM$.\\
(i) $\E_{\bM}(\frac{\bw_{1}^{\top} h(\bx)}{\rho})=\pr_{\theta\sim Q}\left(A_{S^1} ~ | ~ \theta^{\top}\bx\geq \tau_{\bx}(\rho/m)\right)$\\
(ii) $\E_{\bM}(\frac{\bw_{0}^{\top}h(\bx)}{\rho})=\pr_{\theta\sim Q}\left(A_{S^0} ~ | ~ \theta^{\top}\bx\geq \tau_{\bx}(\rho/m)\right)$\\
(iii) $\E_{\bM}(\frac{\bw_{1}^{\top} h(\bx)}{\rho})\geq 1-\sum_{\bx'\in S^{1}}q(\bx,\bx')$\\
(iv) $\E_{\bM}(\frac{\bw_{1}^{\top} h(\bx)}{\rho})\leq 1-q(\bx,\bx_{*}^1)$\\
(v) $\E_{\bM}(\frac{\bw_{0}^{\top} h(\bx)}{\rho})\geq 1-\sum_{\bx'\in S^{0}}q(\bx,\bx')$\\
(vi) $\E_{\bM}(\frac{\bw_{0}^{\top} h(\bx)}{\rho})\leq 1-q(\bx,\bx_{*}^0)$
\end{lemma}
\begin{proof}
Part (i) and (ii) follows from simple application of Lemma~\ref{alem:expected_novelty}  to class specific {\fbf}s. Part (iii) and (v) follows from simple application of Lemma~\ref{alem:bound_novelty} to class specific {\fbf}s. For part (iv), simple application of Lemma~\ref{alem:bound_novelty} to \fbf $\bw_{1}$ ensures  that for any $\bx'\in S^{1}, \E_{\bM}(\frac{\bw_{1}^{\top} h(\bx)}{\rho})\leq 1-q(\bx,\bx')$. Clearly, $\E_{\bM}(\frac{\bw_{1}^{\top} h(\bx)}{\rho})\leq 1-q(\bx,\bx_{*}^1)$. Applying similar argument, part (vi) also holds.
\end{proof}

\begin{lemma}\label{lem:concentration}
Let $\bw_0$ and $\bw_1$ be the $\fbf$s constructed using $S^0$ and $S^1$. For any $\bx\in \R^d$ let $\mu_0=\E_{\bM}\left(\frac{\bw_0^{\top} h(\bx)}{\rho}\right)$ and $\mu_1=\E_{\bM}\left(\frac{\bw_1^{\top} h(\bx)}{\rho}\right)$ Then, for any $\epsilon>0$ the following holds,\\
(i) $\pr\left(\frac{\bw_1^{\top} h(\bx)}{\rho}\geq (1+\epsilon)\mu_1\right)\leq\exp(-\epsilon^2\rho\mu_1/3)$ \\
(ii) $\pr\left(\frac{\bw_0^{\top} h(\bx)}{\rho}\geq (1+\epsilon)\mu_0\right)\leq\exp(-\epsilon^2\rho\mu_0/3)$\\
(iii) $\pr\left(\frac{\bw_1^{\top} h(\bx)}{\rho}\leq (1-\epsilon)\mu_1\right)\leq\exp(-\epsilon^2\rho\mu_1/2)$\\
(iii) $\pr\left(\frac{\bw_0^{\top} h(\bx)}{\rho}\leq (1-\epsilon)\mu_0\right)\leq\exp(-\epsilon^2\rho\mu_0/2)$
\end{lemma}
\begin{proof}
We will only prove part (i) since part (ii) is similar. Let $\Bh=h(\bx), \Bh_1=h(\bx_1),\ldots,\Bh_{n_1}=h(\bx_{n_1})$ be the $\fh$es of $\bx,\bx_1,\ldots, \bx_{n_1}$ that belongs to $S^1$. Define random variables $U_1,\ldots, U_m\in \{0,1\}$ as follows:
\[
    U_j=
\begin{cases}
    1,& \text{if } \Bh_1[j]=\cdots=\Bh_{n_1}[j]=0 \text{ and } \Bh[j]=1\\
    0,              & \text{otherwise}
\end{cases}
\]
The $U_j$ are i.i.d. and
\begin{eqnarray*}
\E_{\bM}(U_j)\hspace{-0.12in}&=&\hspace{-0.12in}\pr_{\bM}(\Bh[j]=1)\times\\
&&\hspace{-0.12in}\pr_{\bM}\left(\Bh_1[j]=\cdots=\Bh_{n_1}[j]=0 ~ | ~ \Bh[j]=1\right)\\
&=&\hspace{-0.12in}\frac{\rho}{m}\E_{\bM}\left(\frac{\bw_1^{\top} h(\bx)}{\rho}\right)
\end{eqnarray*}
where we have used the fact that $\pr_{\bM}(h(\bx)_j=1)=\pr_{\theta\sim Q}(\theta^{\top}\bx\geq \tau_{\bx}(\rho/m))=\rho/m$ and using Lemma 2 of the supplementary material of~\cite{dasgupta2018neural}, $\pr_{\bM}\left(\Bh_1[j]=\cdots=\Bh_{n_1}[j]=0 ~ | ~ \Bh[j]=1\right)=\E_{\bM}\left(\frac{\bw_1^{\top} h(\bx)}{\rho}\right)$. Therefore, $\E_{\bM}(U_1+\cdots+U_m)=\rho\cdot\E_{\bM}\left(\frac{\bw_1^{\top} h(\bx)}{\rho}\right)$. Let $\mu_1=\E_{\bM}\left(\frac{\bw_1^{\top} h(\bx)}{\rho}\right)$. By multiplicative Chernoff bound for any $0<\epsilon<1$, we have,
\begin{align*}
\pr_M\left(U_1+\cdots+U_m\geq (1+\epsilon)\rho\mu_1\right) & \leq\exp(-\epsilon^2\rho\mu_1/3) \\
\pr_M\left(U_1+\cdots+U_m\leq (1-\epsilon)\rho\mu_1\right)& \leq\exp(-\epsilon^2\rho\mu_1/2)
\end{align*}
Noticing that $U_1+\cdots+U_m=\bw_1^{\top}h(\bx)$, the result follows.
\end{proof}

\subsection{Proof of Theorem~\ref{th:knnc}}

\begin{proof}
Without loss of generality, assume that \knnc prediction of $\bx$ is 1. For the case when \knnc prediction is 0 is similar.
Prediction of \fnn on $\bx$ agrees with the prediction of \knnc whenever $\left(\bw_1^{\top}h(\bx)/\rho\right)<\left(\bw_0^{\top}h(\bx)/\rho\right)$. We first show that $\E_{\bM}\left(\bw_1^{\top}h(\bx)/\rho\right)<\E_{\bM}\left(\bw_0^{\top}h(\bx)/\rho\right)$ with high probability and then using standard concentration bound presented in lemma \ref{lem:concentration}, we achieve the desired result.

Since $\|\bx-\bx_{*}\|\leq\frac{\eta}{2}$, all the $\lceil\frac{k+1}{2}\rceil$ nearest neighbors of $\bx$ have same label. Let $\|\bx-\bx_{*}\|_{\infty}\leq \Delta/s$, where $\Delta=\frac{1}{2}\left(\tau_{\bx}(\rho/2m)-\tau_{\bx}(\rho/m)\right)$. Using lemma~\ref{alem:q_lowerbound}, we get $q(\bx,\bx_{*})\geq 1/2$. Combining this with part (iv) of lemma \ref{alem:expectation}, we get $\mu_1=\E_{\bM}\left(\bw_1^{\top}h(\bx)/\rho\right)\leq 1-q(\bx,\bx^*)\leq 1/2$.

If $\bx$ is sampled from a permutation invariant distribution, using corollary~\ref{acor:expected_q}, we get $\E_{\bx} q(\bx,\bx_i)=\rho/m$ for each $\bx'\in S^0$, and thus using linearity of expectation, $\E_{\bx}\left(\sum_{\bx'\in S^0}q(\bx,\bx')\right)=\sum_{\bx'\in S^0}\E_{\bx} q(\bx,\bx')=\rho n_0/m$. For any $\alpha>0$,  using Markov's inequality,
\begin{equation}
pr\left(\sum_{\bx'\in S^0}q(\bx,\bx')>\alpha\right)\leq\frac{\E_{\bx}\left(\sum_{\bx'\in S^0}q(\bx,\bx'\right)}{\alpha}=\frac{\rho n_0}{m\alpha}.
\end{equation}

Specifically, choose $\alpha=1/4$. Then with probability $\geq 1-\frac{4\rho n_0}{m}$, we have $\sum_{\bx'\in S^0}q(\bx,\bx')\leq \frac{1}{4}$. Combining this with part (v) of lemma \ref{alem:expectation}, we immediately get, with probability $\geq 1-\frac{4\rho n_0}{m}$, we have $\mu_0=\E_{\bM}\left(\bw_0^{\top}h(\bx)/\rho\right)\geq 3/4$.

Next we show that $\frac{\bw_1^{\top}h(\bx)}{\rho}\leq 3/5$ with high probability.

 If we set $\epsilon=\frac{1}{10\mu_1}$, then we get $(1+\epsilon)\mu_1=\mu_1+\epsilon\mu_1\leq\frac{1}{2}+\epsilon\mu_1=\frac{1}{2}+\frac{1}{10}=\frac{3}{5}$. For this choice of $\epsilon$, $\mu_1(1+\epsilon)\leq 3/5$. Therefore,
\begin{eqnarray*}
\pr\left(\frac{w_1^{\top}h(\bx)}{\rho}> 3/5\right)&\leq&\pr\left(\frac{w_1^{\top}h(\bx)}{\rho}> (1+\epsilon)\mu_1\right)\\
&\leq&\exp(-\epsilon^2\rho\mu_1/3)\\
&=&\exp\left(-\frac{\rho}{300\mu_1}\right)\\
&\leq&\exp\left(-\frac{\rho}{150}\right)
\end{eqnarray*}
where the first inequality follows from our choice of $\epsilon$, the second inequality follows from Lemma~\ref{lem:concentration}, the equality follows from our choice of $\epsilon$ and the third inequality follows since $\mu_1\leq \frac{1}{2}$.

Next, we would like to show $\frac{\bw_0^{\top}h(\bx)}{\rho}\geq\frac{5}{8}>\frac{3}{5}$ with high probability. If we set set $\epsilon_1=\frac{1}{6}$ and using the fact that $\mu_0\geq \frac{3}{4}$, we get $(1-\epsilon_1)\mu_0=\frac{5}{6}\mu_0\geq\frac{5}{6}\cdot\frac{3}{4}=\frac{5}{8}$.

Now we have,
\begin{eqnarray*}
\pr\left(\frac{w_0^{\top}h(\bx)}{\rho}<5/8\right)&\leq&\pr\left(\frac{w_0^{\top}h(\bx)}{\rho}< (1-\epsilon_1)\mu_1\right)\\
&\leq&\exp(-\epsilon^2\rho\mu_0/2)\\
&=&\exp\left(-\frac{\rho\mu_0}{72}\right)\\
&\leq&\exp\left(-\frac{\rho}{96}\right)\\
\end{eqnarray*}
where the first inequality follows from our choice of $\epsilon_1$, the second inequality follows from Lemma~\ref{lem:concentration}, the equality follows again from our choice of $\epsilon_1$ and the third inequality follows from the fact that $\mu_1\leq 1$.

Therefore, with probability at least $1-\left(\frac{4\rho n_0}{m}+e^{-\frac{\rho}{150}}+e^{-\frac{\rho}{96}}\right)$ we have,  (i) $\frac{\bw_0^{\top}h(\bx)}{\rho}\geq\frac{5}{8}$, and (ii) $\frac{3}{5}\geq\frac{\bw_1^{\top}h(\bx)}{\rho}$. Since $\frac{5}{8}>\frac{3}{5}$, the result follows.
\end{proof}

\section{Analyses for Algorithm~\ref{alg:dist-train-infer}}
\label{asec:alg2-proofs}

The proofs for all the theoretical results in \S \ref{sec:algo:dist-fbfc} are presented in this section.

\subsection{Proof of Theorem~\ref{thm:parity-1}}

\begin{proof}
Given the per-party data chunk $S_t, t \in [\tau]$, let us consider
the \fbf $\bw_l \in (0,1)^m$ for any class $l \in [L]$ learned over
the pooled data $S = \cup_{t \in [\tau]} S_t$ using the Train\fnn
subroutine in Algorithm~\ref{alg:train-infer}. For any $i \in [m]$ and
$l \in [L]$:

\begin{align}
  \bw_l[i] & = 1 \cdot \underbrace{
    \gamma \cdot \gamma \cdot \cdots \cdot \gamma
  }_{\text{\# times } \Bh[i] = 1 \text{ for some } (\bx,y) \in S, y = l} \nonumber
  \\
  & = \gamma^{\left| \left\{
    (\bx,y) \in S: y = l, \Bh = \Gamma_\rho(\bM \bx), \Bh[i] = 1
  \right\} \right|}. \label{aeq:pooled-data-fbf}
\end{align}

By the same argument, the \fbf $\bw_l^t$ learned with data chunk $S_t$
for a class $l$ (Algorithm~\ref{alg:dist-train-infer}, line 4) can be
summarized as follows for any $i \in [m]$:

\begin{align}
  \bw_l^t[i] & = \gamma^{\left| \left\{
    (\bx,y) \in S_t: y = l, \Bh = \Gamma_\rho(\bM \bx), \Bh[i] = 1
  \right\} \right|}. \label{aeq:data-chunk-fbf}
\end{align}

Then the all-reduced \fbf $\hat{\bw}_l$ for a class $l$
(Algorithm~\ref{alg:dist-train-infer}, line 9) is given by the
following for any $i \in [m]$:

\begin{align}
  \hat{\bw}_l[i] & = \gamma^{
    \sum_{t \in [\tau]} \log_\gamma \bw_l^t[i]
  }
  \\
  & \stackrel{(A)}{=} \gamma^{
    \sum_{t \in [\tau]} \left| \left\{
      (\bx,y) \in S_t: y = l, \Bh = \Gamma_\rho(\bM \bx), \Bh[i] = 1
    \right\} \right|
  }
  \\
  & \stackrel{(B)}{=} \gamma^{
    \left| \cup_{t \in [\tau]} \left\{
      (\bx,y) \in S_t: y = l, \Bh = \Gamma_\rho(\bM \bx), \Bh[i] = 1
    \right\} \right|
  }
  \\
  & \stackrel{(C)}{=} \gamma^{
    \left| \left\{
      (\bx,y) \in \underbrace{\cup_{t \in [\tau]} S_t}_{S}:
      y = l, \Bh = \Gamma_\rho(\bM \bx), \Bh[i] = 1
    \right\} \right|
  }
  \\
  & = \gamma^{
    \left| \left\{
      (\bx,y) \in S:
      y = l, \Bh = \Gamma_\rho(\bM \bx), \Bh[i] = 1
    \right\} \right|
  }
  \\
  & \stackrel{(D)}{=} \bw_l[i],
\end{align}

where $(A)$ is obtained from \eqref{aeq:data-chunk-fbf}, $(B)$ is
obtained from the fact that the sets $S_t, t \in [\tau]$ are disjoint
(no data shared between parties), $(C)$ is from the fact that a union
of subsets of disjoint sets ($S_t$) is the same as a subset of union
of disjoint sets $\cup_t S_t$. $(D)$ follows from
\eqref{aeq:pooled-data-fbf}.

Since the above holds for all $i \in [m]$, we can say that $\bw_l =
\hat{\bw}_l \forall l \in [L]$, proving the statement of the claim.
\end{proof}

\subsection{Proof of Lemma~\ref{thm:alg3-train}}
\label{asec:alg3-train-proof}

\begin{proof}
We begin with recalling that the all-reduce operation can be performed
efficiently in a peer-to-peer communication setup where the $\tau$
parties can be organized as a binary tree of depth $O( \log \tau
)$. Then the communication at each level of the tree can be done in
parallel for each independent subtree at that level. Consider the
object being all-reduced to be of size $O(c)$. Then in the first round
of communication, $O(\tau/2)$ pairs of parties combine their objects
in parallel in time $O(c)$ with total communication $O(c \tau/2)$ and
$O(c)$ memory overhead in each of the parties. In the second round,
$O(\tau/4)$ pairs of parties combine their objects in parallel again
in time $O(c)$ with total communication $O(c \tau / 4)$ with $O(c)$
memory overhead in $O(\tau/2)$ of the parties. Going up the tree to
the root then takes time $O(c \log \tau)$. The total communication
cost is $O(c \tau)$.

The communication first goes bottom up from the leaves to the root,
which then has the final all-reduced result. Then this result is sent
to each party top-down from the root (in a corresponding manner) so
that eventually all parties have the all-reduced result in time $O(c
\log \tau)$ with total $O(c \tau)$ communication.

Based on the above complexities of the all-reduce operation, we can
summarize the complexities for the different operations in
Train{\fnn}FLDP (Algorithm~\ref{alg:dist-train-infer})  when DP is disabled ($\tt{IS\_DP}=false$) as follows:
\begin{itemize}
\item The broadcast of the random seed in line 2 can be done with an
  all-reduce in $O(\log \tau)$ time and $O(\tau)$ total communication
  and $O(1)$ memory overhead in each party.
\item On party $V_t$, the invocation of Train\fnn in line 4 on data
  chunk $S_t$ of size $n_t$ takes $O(n_t \cdot m \cdot \max \{s, \log
  \rho\})$ time and $O(m \cdot \max\{s, L\})$ memory from
  Lemma~\ref{thm:alg1-train} and no communication cost.
\item The all-reduce of the per-party per-class {\fbf}s in line 9
  takes time $O( m \cdot L \cdot \log \tau )$ with $O( m \cdot L \cdot
  \tau )$ total communication, and $O( m \cdot L )$ memory overhead
  per-party.
\end{itemize}

Putting them all together gives us the per-party time complexity of
$O( n_t \cdot m \cdot \max \{ s, \log \rho, (L/n_t) \log \tau \} )$,
memory overhead of $O( m \cdot \max\{ s, L \} )$ and total
communication among all parties of $O( m \cdot L \cdot \tau )$, giving
us the statement of the claim.

When DP is enabled ($\tt{IS\_DP}=true$), we can summarize the complexity of various operations as follows, 
\begin{itemize}
\item The broadcast of the random seed in line 2 can be done with an
  all-reduce in $O(\log \tau)$ time and $O(\tau)$ total communication
  and $O(1)$ memory overhead in each party.
\item On party $V_t$, the invocation of Train\fnn in line 4 on data
  chunk $S_t$ of size $n_t$ takes $O(n_t \cdot m \cdot \max \{s, \log
  \rho\})$ time and $O(m \cdot \max\{s, L\})$ memory from
  Lemma~\ref{thm:alg1-train} and no communication cost.
\item On party $V_t$, the invocation of DP subroutine in line 6 takes $O(m \cdot L \cdot T)$ time, $O(m\cdot L)$ memory  and requires no communication cost.
\item The all-reduce of the per-party per-class {\fbf}s in line 9
  takes time $O( T \cdot \log \tau )$ with $O( T \cdot
  \tau )$ total communication, and $O( m \cdot L )$ memory overhead
  per-party.
\end{itemize}
Putting them all together gives us the per-party time complexity of
$O( n_t \cdot m \cdot \max \{ s, \log \rho, (LT/n_t)\} +T\cdot \log \tau)$,
memory overhead of $O( m \cdot \max\{ s, L \} )$ and total
communication among all parties of $O( T \cdot \tau )$, giving
us the statement of the claim.
\end{proof}

\subsection{Proof of Theorem \ref{th:dp_proof}}
\begin{proof}
Let $\epsilon_0=\frac{\epsilon}{2T\tau}$. Note that a single data point (record) $\bx$ can affect a single \fbf corresponding to its own class labels located at a fixed party. For any pair of neighboring databases $D$ and $\bar{D}$ that differ in a single record, let $\bc$ and $\bar{\bc}$ be the respective count vectors (concatenating $L$ count vectors, one per class label).
Then for any $i\in [m\times L],~ |\bc[i]-\bar{\bc}[i]|\leq 1$. Fix any party $t\in[\tau]$ and any iteration  $j\in [T]$. In line 17 of Algorithm \ref{alg:dist-train-infer}, $i_j$ is sampled using the exponential mechanism in iteration $j$, therefore using standard properties of exponential mechanism, releasing the index  $i_j$ is $\epsilon_0$ differentially private. Next using the property of Laplace mechanism, releasing the value at the $i_j^{th}$ index, i.e.,  $\bc[i_j]$ (in line 18-19 of Algorithm \ref{alg:dist-train-infer}) is $\epsilon_0$ differentially private. Therefore, releasing the count $\bc[i_j]$  is $\epsilon_0+\epsilon_0=2\epsilon_0=\frac{\epsilon}{T\tau}$ difefrentially private. Applying the composition theorem over all $t\in [\tau]$ and  $j\in[T]$,  Train{\fnn}\hc{FLDP} is $(\epsilon,0)$ differentially private.
\end{proof}

\section{Empirical evaluations} \label{asec:emp-eval}

\subsection{Implementation \& Compute Resources details}
\label{asec:emp-eval:details}

The proposed \fnn is implemented in Python 3.6 to fit the \sklearn
API~\citep{pedregosa2011scikit}, but the current implementation is not
optimized for computational performance. We use the \sklearn
implementation of \knnc. The experiments are performed on a 16-core
128GB machine running Ubuntu 18.04. The code is available at \url{https://github.com/rithram/flynn}.

\subsection{Synthetic data description} \label{asec:emp-eval:syndata}

We consider {\bf synthetic data} of varying sizes and
properties. These synthetic data are designed in a way that favors
local classifiers such as the nearest-neighbor classifiers -- each
class conditional data distribution consists of multiple separated
modes, with enough separation between modes of different
classes~\citep{guyon2003design}. We consider 5 classes with 3 modes
per class. These data are balanced in the class sizes and have no
labeling noise. We consider these synthetic data to see if our
proposed \fnn is able to capture multiple separated local class
neighborhoods in a single per-class \fbf encoding while providing
enough separation between {\fbf}s of different classes to have strong
classification performance. To generate synthetic datasets, we use
the \texttt{data.make\_classification}
functionality~\citep{guyon2003design} in
\sklearn~\citep{pedregosa2011scikit}.

We also study the effect of the number of non-zeros $b < d$ in the
binary data on the performance of \fnn and baselines. The results
indicate that, for fixed data dimensionality $d$, the relative
performance of \fnn is not significantly affected by the choice of $b
< d$. The performance of \sbfc seems to improve with increasing $b$.

\begin{table}[htb]
\caption{Comparison of \fnn with nearest neighbor classifiers and
  \sbfc on synthetic data. For binary data $\mathcal X = \HC^d$, we
  try different values of number of non-zeros $b < d$, denoted by
  different values of $|x| = b, x \in \mathcal X$. We report the
  normalized accuracy aggregated over 30 random synthetic datasets.
  Normalized accuracy of \knnc is zero hence not present in the table.
}
\label{atab:syn-data}
\begin{center}
{\small
\begin{tabular}{ccccc}
\toprule
$\mathcal X$      & $n$  & $d$  & \sbfc & \fnn \\
\midrule
$\HC^d, |x| = 10$ & $10^3$ & $100$ & $0.63 \pm 0.02$ & $-0.06 \pm 0.02$ \\
$\HC^d, |x| = 20$ & $10^3$ & $100$ & $0.58 \pm 0.03$ & $-0.04 \pm 0.02$ \\
$\HC^d, |x| = 30$ & $10^3$ & $100$ & $0.29 \pm 0.04$ & $-0.03 \pm 0.02$ \\
$\HC^d, |x| = 40$ & $10^3$ & $100$ & $0.18 \pm 0.03$ & $-0.03 \pm 0.02$ \\
\bottomrule
\end{tabular}
}
\end{center}
\end{table}

\subsection{OpenML data details} \label{asec:emp-eval:openml}

We consider OpenML datasets utilizing the following query for OpenML
classification datasets with no categorical and missing features with
number of features $d \in [20, 1000]$, number of rows $n \in [1000,
  20000]$ and number of classes $L \in [2, 30]$, leading to $70$ data
sets where there were no issues with the data retrieval and the
processing of the data with \sklearn operators.

We deliberately chose a large set of datasets to evaluate how
generally \fnn is able to mimic \knnc. While our theoretical
guarantees require certain margin conditions, we wanted to look at a
wide variety of datasets where such conditions may or may not be
satisfied and understand the true empirical behaviour of \fnn relative
to \knnc.

\paragraph{OpenML query for datasets.}
We use the following code snippet to obtain the list of datasets we
try. Details on the precise 70 datasets we used is presented in
Table~\ref{atab:openml-data}.
\begin{lstlisting}[language=Python]
from openml.datasets import list_datasets, get_dataset
openml_df = list_datasets(output_format='dataframe')
val_dsets = openml_df.query(
    'NumberOfInstancesWithMissingValues == 0 & '
    'NumberOfMissingValues == 0 & '
    'NumberOfClasses > 1 & '
    'NumberOfClasses <= 30 & '
    'NumberOfSymbolicFeatures == 1 & '
    'NumberOfInstances > 999 &'
    'NumberOfFeatures >= 20 &'
    'NumberOfFeatures <= 1000 &'
    'NumberOfInstances <= 20000'
)[[
    'name', 'did', 'NumberOfClasses',
    'NumberOfInstances', 'NumberOfFeatures'
]]
\end{lstlisting}

\begin{table}[!htb]
\caption{OpenML datasets used in this evaluation.}
\label{atab:openml-data}
\begin{center}
{\tiny
\begin{tabular}{lcccc}
\toprule
Name & OpenML ID & $L$ & $n$ & $d$ \\
\midrule
mfeat-fourier & 14 & 10 & 2000 & 77 \\
mfeat-karhunen & 16 & 10 & 2000 & 65 \\
mfeat-zernike & 22 & 10 & 2000 & 48 \\
optdigits & 28 & 10 & 5620 & 65 \\
spambase & 44 & 2 & 4601 & 58 \\
waveform-5000 & 60 & 3 & 5000 & 41 \\
satimage & 182 & 6 & 6430 & 37 \\
fri-c3-1000-25 & 715 & 2 & 1000 & 26 \\
fri-c4-1000-100 & 718 & 2 & 1000 & 101 \\
pol & 722 & 2 & 15000 & 49 \\
fri-c4-1000-25 & 723 & 2 & 1000 & 26 \\
ailerons & 734 & 2 & 13750 & 41 \\
puma32H & 752 & 2 & 8192 & 33 \\
cpu-act & 761 & 2 & 8192 & 22 \\
fri-c4-1000-50 & 797 & 2 & 1000 & 51 \\
fri-c3-1000-50 & 806 & 2 & 1000 & 51 \\
bank32nh & 833 & 2 & 8192 & 33 \\
fri-c1-1000-50 & 837 & 2 & 1000 & 51 \\
fri-c0-1000-25 & 849 & 2 & 1000 & 26 \\
fri-c2-1000-50 & 866 & 2 & 1000 & 51 \\
fri-c2-1000-25 & 903 & 2 & 1000 & 26 \\
fri-c0-1000-50 & 904 & 2 & 1000 & 51 \\
fri-c1-1000-25 & 917 & 2 & 1000 & 26 \\
mfeat-fourier & 971 & 2 & 2000 & 77 \\
waveform-5000 & 979 & 2 & 5000 & 41 \\
optdigits & 980 & 2 & 5620 & 65 \\
mfeat-zernike & 995 & 2 & 2000 & 48 \\
mfeat-karhunen & 1020 & 2 & 2000 & 65 \\
pc4 & 1049 & 2 & 1458 & 38 \\
pc3 & 1050 & 2 & 1563 & 38 \\
kc1 & 1067 & 2 & 2109 & 22 \\
pc1 & 1068 & 2 & 1109 & 22 \\
PizzaCutter3 & 1444 & 2 & 1043 & 38 \\
PieChart3 & 1453 & 2 & 1077 & 38 \\
cardiotocography & 1466 & 10 & 2126 & 36 \\
first-order-theorem-proving & 1475 & 6 & 6118 & 52 \\
hill-valley & 1479 & 2 & 1212 & 101 \\
ozone-level-8hr & 1487 & 2 & 2534 & 73 \\
qsar-biodeg & 1494 & 2 & 1055 & 42 \\
ringnorm & 1496 & 2 & 7400 & 21 \\
wall-robot-navigation & 1497 & 4 & 5456 & 25 \\
steel-plates-fault & 1504 & 2 & 1941 & 34 \\
twonorm & 1507 & 2 & 7400 & 21 \\
autoUniv-au1-1000 & 1547 & 2 & 1000 & 21 \\
cardiotocography & 1560 & 3 & 2126 & 36 \\
hill-valley & 1566 & 2 & 1212 & 101 \\
GesturePhaseSegmentationProcessed & 4538 & 5 & 9873 & 33 \\
thyroid-ann & 40497 & 3 & 3772 & 22 \\
texture & 40499 & 11 & 5500 & 41 \\
Satellite & 40900 & 2 & 5100 & 37 \\
steel-plates-fault & 40982 & 7 & 1941 & 28 \\
sylvine & 41146 & 2 & 5124 & 21 \\
ada & 41156 & 2 & 4147 & 49 \\
microaggregation2 & 41671 & 5 & 20000 & 21 \\
Sick-numeric & 41946 & 2 & 3772 & 30 \\
mfeat-factors & 12 & 10 & 2000 & 217 \\
isolet & 300 & 26 & 7797 & 618 \\
mfeat-factors & 978 & 2 & 2000 & 217 \\
gina-agnostic & 1038 & 2 & 3468 & 971 \\
gina-prior2 & 1041 & 10 & 3468 & 785 \\
gina-prior & 1042 & 2 & 3468 & 785 \\
cnae-9 & 1468 & 9 & 1080 & 857 \\
madelon & 1485 & 2 & 2600 & 501 \\
semeion & 1501 & 10 & 1593 & 257 \\
clean2 & 40666 & 2 & 6598 & 169 \\
Speech & 40910 & 2 & 3686 & 401 \\
mfeat-pixel & 40979 & 10 & 2000 & 241 \\
USPS & 41082 & 10 & 9298 & 257 \\
madeline & 41144 & 2 & 3140 & 260 \\
philippine & 41145 & 2 & 5832 & 309 \\
gina & 41158 & 2 & 3153 & 971 \\
\bottomrule
\end{tabular}
}
\end{center}
\end{table}

\paragraph{License.}
The license regarding the OpenML platform and all the empirical data
and metadata are discussed in \url{https://www.openml.org/cite}. The
empirical data and metadata are free to use under CC-BY license. The
OpenML platform and libraries are BSD licensed.

\subsection{Details for {\fnn}FL scaling} \label{asec:emp-eval:scaling}

We consider the 6 datasets for evaluating the scaling of Train{\fnn}FL
(Algorithm~\ref{alg:dist-train-infer}) with the number of parties
$\tau$, when the data is evenly split between all parties. The details
regarding the datasets are provided in Table~\ref{atab:data-stats}.

\begin{table}[!htb]
\caption{Details of a subset of the datasets. For CIFAR-10 and
  CIFAR-100, we collapse the 3 color channels and then flatten the $32
  \times 32$ images to points in $\Real^{1024}$. For MNIST and
  Fashion-MNIST, we flatten the $28 \times 28$ images to points in
  $\Real^{784}$.}
\label{atab:data-stats}
\begin{center}
{\small
\begin{tabular}{lcccc}
\toprule
Dataset & $n$ & $d$ & $L$ & Obtained from \\
\midrule
Digits        & $1797$  & $64$   & $10$  & OpenML\\
Letters       & $20000$ & $16$   & $26$  & OpenML\\
MNIST         & $60000$ & $784$  & $10$  & Tensorflow\\
Fashion-MNIST & $60000$ & $784$  & $10$  & Tensorflow\\
CIFAR-10      & $50000$ & $1024$ & $10$  & Tensorflow\\
CIFAR-100     & $50000$ & $1024$ & $100$ & Tensorflow\\
\bottomrule
\end{tabular}
}
\end{center}
\end{table}

\begin{table}[!htb]
\caption{\fnn hyper-parameters used for the scaling experiment.}
\label{atab:scaling-hp}
\begin{center}
{\small
\begin{tabular}{lcccc}
\toprule
Dataset & $m$ & $s/d$ & $\rho$ & $\gamma$ \\
\midrule
Digits        & $256 \times 64$ & 0.3 & 32 & 0 \\
Letters       & $1447 \times 16$ & 0.5 & 221 & 0.1 \\
MNIST         & $217 \times 784$ & 0.025 & 17 & 0.51 \\
Fashion-MNIST & $138 \times 784$ & 0.105 & 8 & 0.8 \\
CIFAR-10      & $217 \times 1024$ & 0.026 & 26 & 0.51 \\
CIFAR-100     & $217 \times 1024$ & 0.026 & 26 & 0.51 \\
\bottomrule
\end{tabular}
}
\end{center}
\end{table}

\paragraph{License.}
The license for the OpenML datasets are discussed in \S
\ref{asec:emp-eval:openml}. We obtain the widely used vision datasets
from \texttt{tensorflow.keras.datasets} module in
Tensorflow~\citep{abadi2016tensorflow}.

\paragraph{Raw runtimes.}
We also present the raw runtimes that were used to generate the
speedup plot in Figure~\ref{fig:scaling} in Table~\ref{atab:scaling}.

\begin{table}[!htb]
\caption{Raw runtimes $T$ (in seconds) and speedups $S$.}
\label{atab:scaling}
\begin{center}
{\tiny
\begin{tabular}{lccccccccc}
\toprule
Dataset & $T$($\tau = 1$) & $T$($\tau = 2$) & $S$($\tau = 2$) & $T$($\tau = 4$) & $S$($\tau = 4$) & $T$($\tau = 8$) & $S$($\tau = 8$) & $T$($\tau = 16$) & $S$($\tau = 16$) \\
\midrule
Digits & 3.63$\pm$0.06 & 3.26$\pm$0.10 & 1.11$\pm$0.04 & 1.84$\pm$0.06 & 1.97$\pm$0.07 & 1.36$\pm$0.05 & 2.67$\pm$0.11 & 1.16$\pm$0.03 & 3.12$\pm$0.06 \\
Letter & 25.72$\pm$0.41 & 14.74$\pm$0.42 & 1.75$\pm$0.04 & 7.72$\pm$0.20 & 3.33$\pm$0.11 & 4.06$\pm$0.33 & 6.38$\pm$0.50 & 2.91$\pm$0.07 & 8.85$\pm$0.25 \\
MNIST & 1023.59$\pm$14.88 & 518.85$\pm$8.19 & 1.97$\pm$0.04 & 262.68$\pm$5.98 & 3.90$\pm$0.08 & 163.75$\pm$5.99 & 6.26$\pm$0.29 & 122.30$\pm$6.80 & 8.39$\pm$0.44 \\
Fashion-MNIST & 1410.29$\pm$13.66 & 712.98$\pm$3.59 & 1.98$\pm$0.02 & 360.74$\pm$5.47 & 3.91$\pm$0.07 & 241.70$\pm$8.33 & 5.84$\pm$0.18 & 191.15$\pm$1.42 & 7.38$\pm$0.09 \\
CIFAR-10 & 1300.90$\pm$36.63 & 644.99$\pm$4.70 & 2.02$\pm$0.06 & 330.35$\pm$7.21 & 3.94$\pm$0.10 & 207.57$\pm$8.27 & 6.28$\pm$0.33 & 151.67$\pm$3.70 & 8.58$\pm$0.23 \\
CIFAR-100 & 1268.86$\pm$7.85 & 649.20$\pm$5.43 & 1.95$\pm$0.02 & 333.33$\pm$9.36 & 3.81$\pm$0.11 & 211.20$\pm$11.73 & 6.03$\pm$0.33 & 162.80$\pm$1.26 & 7.79$\pm$0.08 \\
\bottomrule
\end{tabular}
}
\end{center}
\end{table}

\subsection{Details on \sbfc baseline} \label{asec:emp-eval:sbfc}

To ablate the effect of the high level of sparsity in \fh, we utilize
the binary \sh~\citep{charikar2002similarity} based locality sensitive
bloom filter for each class in place of \fbf to get \sh Bloom Filter
classifier (\sbfc). \sh is binary like the \fh we consider, however,
it is not explicitly sparse as \fh. In fact, the number of non-zeros
in \fh is $\rho$, while for \sh with dimensionality $m$, in
expectation, we would expect $\approx m/2$ non-zeros in the \sh. We
tune over the \sh projected dimension $m$, considering $m < d$
(traditional regime where \sh is usually employed) and $m > d$ (as in
\fh). For the same $m$, \sh is more costly ($\sim O(md)$ per point)
than \fh ($\sim O(ms + m \log \rho)$) since $s \ll d$, involving a
dense matrix-vector product instead of a sparse matrix-vector one. The
dimensionality of the \sh $m$ is the hyper-parameter we search over --
we consider both projecting down in the range $m \in [1, d]$ (the
traditional use) and projecting up $m \in [d, 2048d]$, where $d$ is
the data dimensionality.

\subsection{Dependence on hyper-parameters}
\label{asec:emp-eval:hpdep}
\begin{table}[!htb]
\caption{Datasets from OpenML used study of \fnn hyper-parameter
  dependence.}
\label{atab:hpdep-data}
\begin{center}
\begin{small}
\begin{tabular}{lccc}
\toprule
dataset & $n$ & $d$ & $L$ \\
\midrule
Digits        & $1797$  & $64$   & $10$   \\
Letters       & $20000$ & $16$   & $26$   \\
Segment       & $2310$  & $19$   & $7$    \\
Gina Prior 2  & $3468$  & $784$  & $10$   \\
USPS          & $9294$  & $256$  & $10$   \\
Madeline      & $3140$  & $259$  & $2$    \\
\bottomrule
\end{tabular}
\end{small}
\end{center}
\end{table}

We study the effect of the different \fnn hyper-parameters: (i) the
\fh dimension $m$, (ii) the NNZ per-row $s \ll d$ in $\bM$, (iii) the
NNZ $\rho$ in the \fh, and (iv) the \fbf decay rate $\gamma$. We
consider $6$ OpenML datasets (see Table~\ref{atab:hpdep-data} for
details on the datasets). For every hyper-parameter setting, we
compute the $10$-fold cross-validated classification accuracy ($1 - $
misclassification rate). We vary each hyper-parameter while fixing the
others. The results for each of the hyper-parameters and datasets are
presented in Figures~\ref{afig:hpdep3-1} \& \ref{afig:hpdep3-2}. We
evaluate the following configurations for the evaluation of each of
the hyper-parameters:

\begin{itemize}
\item {\bf \fh dimension $m$:} We try $10$ values for $m \in [4d,
  4096d]$ with $(s/d) \in \{ 0.1, 0.3 \}$, $\rho \in \{ 8, 32 \}$,
  $\gamma \in \{0. , 0.5\}$.
\item {\bf Projection density $s/d$:} We try $10$ values for $(s/d)
  \in [0.1, 0.8]$ with $m \in \{ 256, 1024 \}$, $\rho \in \{ 8, 32
  \}$, $\gamma \in \{0, 0.5 \}$.
\item {\bf \fh NNZ $\rho$:} We try $10$ values for $\rho \in [4, 256]$
  with $m \in \{ 256, 1024 \}$, $(s/d) \in \{ 0.1, 0.3 \}$, $\gamma
  \in \{ 0, 0.5 \}$.
\item {\bf \fbf decay rate $\gamma$:} We try $10$ values for $\gamma
  \in [0.1, 0.8] $ and $\gamma = 0$ with $m \in \{ 256, 1024 \}$,
  $(s/d) \in \{ 0.1, 0.3 \}$, $\rho \in \{ 8, 32 \}$.
\end{itemize}

\begin{figure*}[thb]
\centering
\begin{subfigure}{0.235\textwidth}
  \centering
  \includegraphics[width=\textwidth]{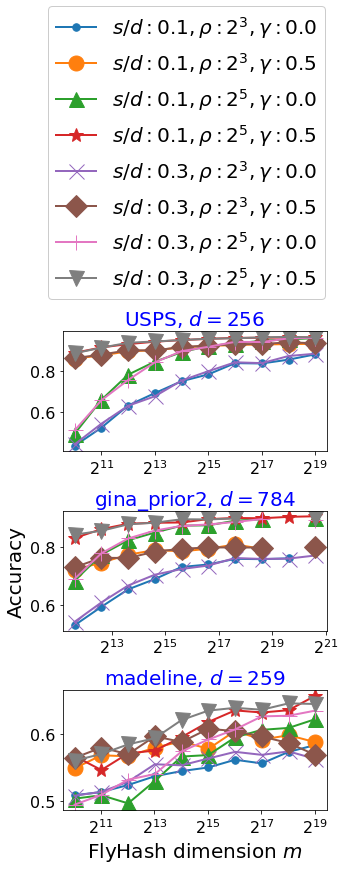}
  \caption{\fh dimension $m$}
  \label{afig:hpdep3-ef-1}
\end{subfigure}
~
\begin{subfigure}{0.235\textwidth}
  \centering
  \includegraphics[width=\textwidth]{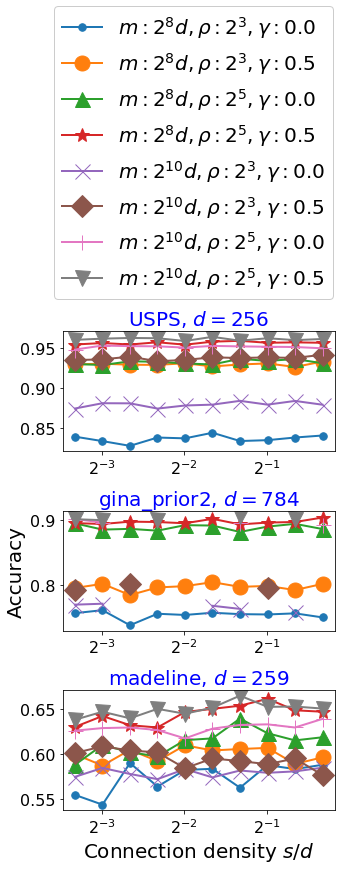}
  \caption{Projection density $s/d$}
  \label{afig:hpdep3-cs-1}
\end{subfigure}
~
\begin{subfigure}{0.235\textwidth}
  \centering
  \includegraphics[width=\textwidth]{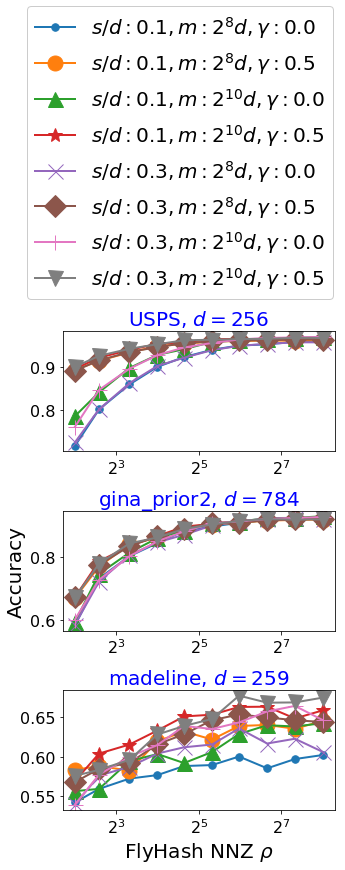}
  \caption{\fh NNZ $\rho$}
  \label{afig:hpdep3-wn-1}
\end{subfigure}
~
\begin{subfigure}{0.235\textwidth}
  \centering
  \includegraphics[width=\textwidth]{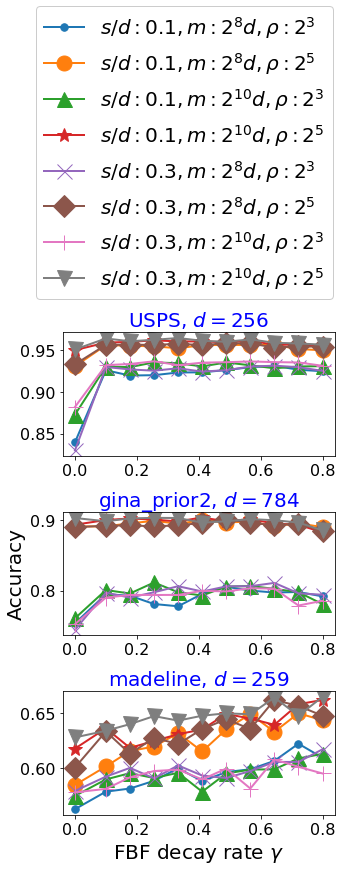}
  \caption{\fbf decay rate $\gamma$}
  \label{afig:hpdep3-c-1}
\end{subfigure}
\caption{{\bf \fnn hyper-parameter dependence -- Part I.} Effect
  of the different \fnn hyper-parameters $m$, $s$, $\rho$,
  $\gamma$ on \fnn predictive performance for $3$ datasets -- the
  horizontal axes correspond to the hyper-parameter being varied
  while fixing the remaining hyper-parameters. The vertical axes
  correspond to the 10-fold cross-validated accuracy for the given
  hyper-parameter configuration ({\em higher is better}). Note the
  log scale on the horizontal axes. {\em Please view in color.} }
\label{afig:hpdep3-1}
\end{figure*}
\begin{figure*}[thb]
\centering
\begin{subfigure}{0.235\textwidth}
  \centering
  \includegraphics[width=\textwidth]{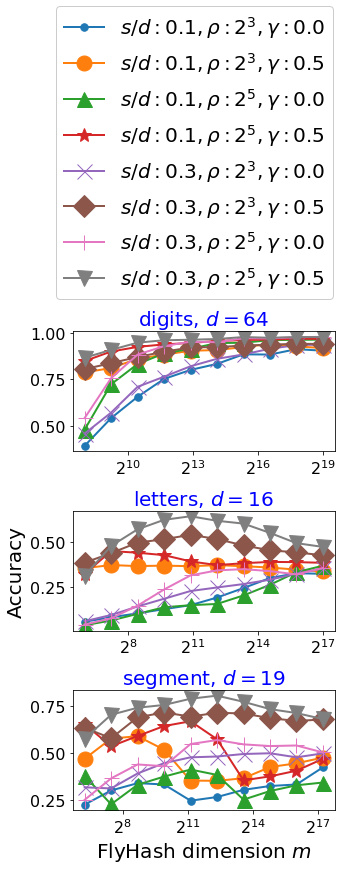}
  \caption{\fh dimension $m$}
  \label{afig:hpdep3-ef-2}
\end{subfigure}
~
\begin{subfigure}{0.235\textwidth}
  \centering
  \includegraphics[width=\textwidth]{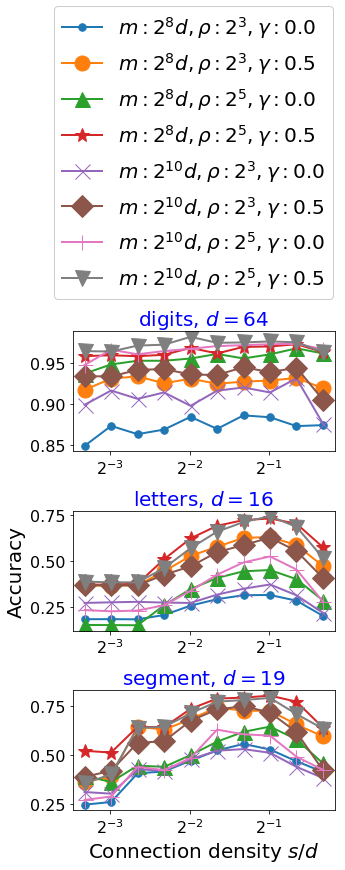}
  \caption{Projection density $s/d$}
  \label{afig:hpdep3-cs-2}
\end{subfigure}
~
\begin{subfigure}{0.235\textwidth}
  \centering
  \includegraphics[width=\textwidth]{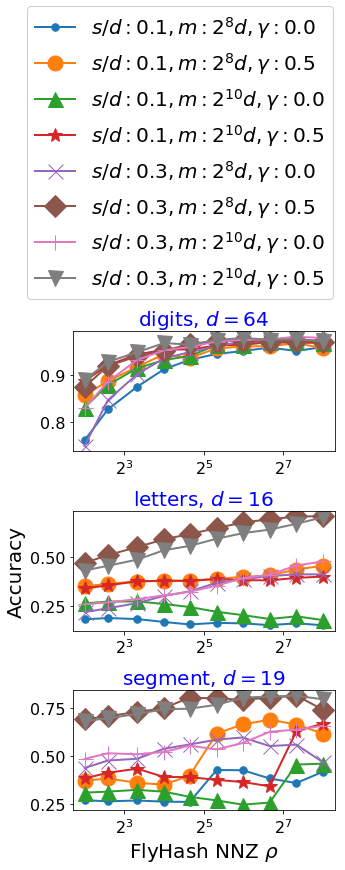}
  \caption{\fh NNZ $\rho$}
  \label{afig:hpdep3-wn-2}
\end{subfigure}
~
\begin{subfigure}{0.235\textwidth}
  \centering
  \includegraphics[width=\textwidth]{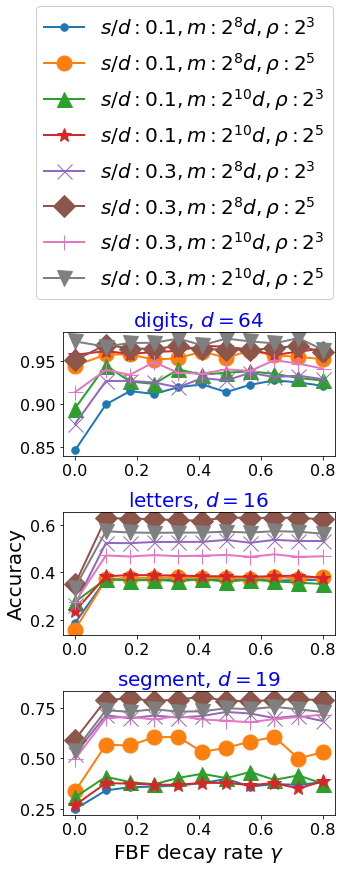}
  \caption{\fbf decay rate $\gamma$}
  \label{afig:hpdep3-c-2}
\end{subfigure}
\caption{{\bf \fnn hyper-parameter dependence -- Part II.} Effect of
  the different \fnn hyper-parameters $m$, $s$, $\rho$, $\gamma$ on
  \fnn predictive performance for $3$ datasets -- the horizontal axes
  correspond to the hyper-parameter being varied while fixing the
  remaining hyper-parameters. The vertical axes correspond to the
  10-fold cross-validated accuracy for the given hyper-parameter
  configuration ({\em higher is better}). Note the log scale on the
  horizontal axes. {\em Please view in color.} }
\label{afig:hpdep3-2}
\end{figure*}

The results in
Figures~\ref{afig:hpdep3-ef-1}~\&~\ref{afig:hpdep3-ef-2} indicate
that, for fixed $\rho$ increasing $m$ improves the \fnn accuracy,
aligning with the theoretical guarantees, up until an upper
bound. This behavior is clear for high dimensional datasets. This
behavior is a bit more erratic for the lower dimensional sets. Larger
values of $m$ improve performance, since it allows us to capture each
class' distribution with smaller random overlap between each class'
{\fbf}s. But the theoretical guarantees also indicate that $\rho$
needs to be large enough, and if $m$ grows too large for any given
$\rho$, the \fnn accuracy might not improve any further.

Figures~\ref{afig:hpdep3-cs-1}~\&~\ref{afig:hpdep3-cs-2} indicate that
for lower dimensional data (such as $d \leq 20$), increasing the
projection density $s$ improves performance up to a point (around $s =
0.5$), after which the performance starts degrading. This is probably
because for smaller values of $s$, not enough information is captured
by the sparse projection for small $d$; for large values of $s$, each
row in the projection matrix $\bM$ become similar to each other,
hurting the similarity-preserving property of \fh. For higher
dimensional datasets, the \fnn performance appears to be somewhat
agnostic to $s$ for any fixed $m$, $\rho$ and $\gamma$.

Figures~\ref{afig:hpdep3-wn-1}~\&~\ref{afig:hpdep3-wn-2} indicate that
increase in $\rho$ leads to improvement in \fnn performance since
large values of $\rho$ better preserve pairwise similarities. However,
if $\rho$ is too large relative to $m$, the sparsity of the subsequent
per-class \fbf go down, thereby leading to more overlap in the
per-class {\fbf}s. So $\rho$ needs to large as per the theoretical
analysis, but not too large.

Figures~\ref{afig:hpdep3-c-1}~\&~\ref{afig:hpdep3-c-2} indicate that
the \fnn is somewhat agnostic to the \fbf decay rate $\gamma$ for any
value strictly greater than $0$ (corresponding to the binary
\fbf). But there is a significant drop in the \fnn performance from
$\gamma > 0$ to $\gamma = 0$ across all dataset -- this behavior is
fairly consistent and apparent.

\subsection{Comparison to gradient descent based baseline} \label{asec:emp-eval:gd-based}

While we can compare nearest-neighbor-based methods (such as \knnc) to gradient-descent-based methods by comparing the best possible optimally tuned performances of each method, it is hard to do such a comparison with some form of parity in terms of computational or communication overhead. We attempt to draw one form of parity based on the communication overhead. Our first proposed scheme for {\fnn}FL has a communication overhead of $O(m)$ where $m$ can be $O(n)$ is the most unfavorable setup. To this end, we wish to compare all schemes for an overall communication budget of $O(n)$. For stochastic gradient-descent based methods, involving communication in every gradient update on a batch of data per-party, a single epoch (a single pass of the data per-party) would correspond to $O(n)$ communication. This is also appropriate for another reason: all the training algorithms -- Train{\fnn} (Algorithm~\ref{alg:train-infer}), Train{\fnn}FLDP (Algorithm~\ref{alg:dist-train-infer}) -- involve a single pass of the training data. Henceforth, we consider two baselines -- logistic regression (\lr) trained with a single epoch, and multi-layered perceptrons (\mlpc) trained with a single epoch. To further remove confounding effects, and obtain the best-case performance for these gradient-descent-based schemes with a single epoch, we consider a centralized training for \lr and \mlpc instead of a true FL training; the performance of the centralized training usually provides an upperbound on the expected performance with FL training.

\paragraph{Hyper-parameter selection.}
For \lr, we consider logistic regression trained for a single epoch with a stochastic algorithm. We utilize the \sklearn implementation (\texttt{linear\_model.LogisticRegression}) and tune over the following hyper-parameters -- (a) penalty type ($\ell_1$/$\ell_2$), (b) regularization $\in \left[ 2^{-10}, 2^{10} \right]$, (c) choice of solver (liblinear~\citep{fan2008liblinear}/SAG~\citep{schmidt2017minimizing}/SAGA~\citep{defazio2014saga}), (d) with/without intercept, (e) one-vs-rest or multinomial for multi-class, (f) with/without class balancing (note that this class balancing operation makes this a two-pass algorithm since we need the first pass to weigh the classes appropriately). We consider a total of 960 hyper-parameter configurations for each of the 70 OpenML dataset.
For \mlpc, we consider a multi-layer perceptron trained for a single epoch with the ``Adam'' stochastic optimization scheme~\citep{kingma2014adam}. We use \texttt{sklearn.neural\_network.MLPClassifier} and tune over the following hyper-parameters -- (a) number of hidden layers $\{1, 2\}$, (b) number of nodes in each hidden layer $\{16, 64, 128\}$, (b) choice of activation function (ReLU/HyperTangent), (d) regularization, (e) batch size $ \in \left[2, 2^8\right]$, (f) initial learning rate $\in \left[10^{-5}, 0.1\right]$ (the rest of the hyper-parameters are left as \sklearn defaults). This leads to a total of 720 hyper-parameters configurations for each of the 70 OpenML datasets.

\begin{figure}
\centering
\includegraphics[width=0.95\textwidth]{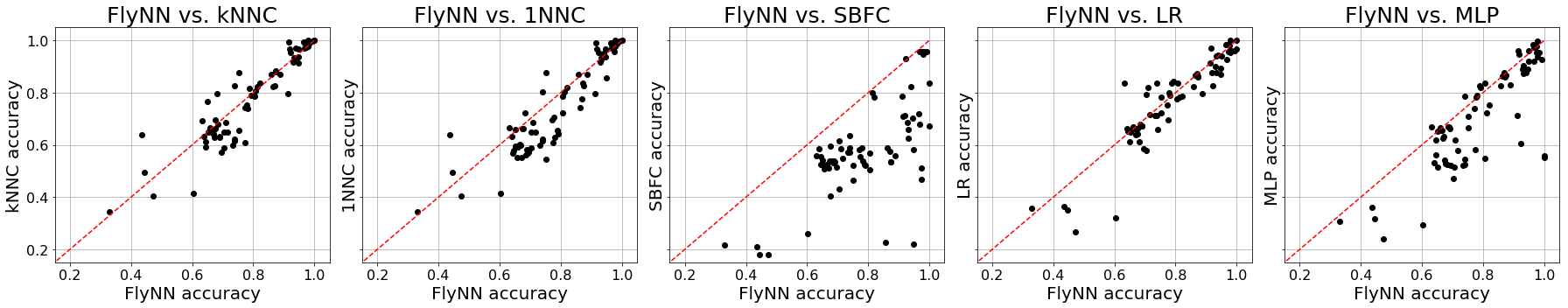}
\caption{Performance of baselines relative to \fnn on {\em OpenML} datasets. The scatter plots compare the best tuned \fnn accuracy against that of \knnc, \onennc, \sbfc, \lr and \mlpc with a point for each dataset, and the red dashed diagonal marking match to \fnn accuracy.}
\label{afig:flynn-base}
\end{figure}
\begin{table}[tb]
\caption{Comparing \fnn to baselines on OpenML datasets with (i) Fraction of datasets \fnn exceed baselines, (ii) Number of datasets on which \fnn has wins(W)/ties(T)/losses(L) over baselines, (iii) Median improvement in normalized accuracy by \fnn over baseline across all datasets, (iv) $p$-values for the paired two-sided t-test (TT), (v) $p$-values for the two-sided Wilcoxon signed rank test (WSRT).}
\label{atab:openml-ind-comp}
\begin{center}
{
\begin{tabular}{lccccc}
\toprule
\textsc{Method} & (i) {\sc Frac.} & (ii) {\sc W/T/L} & (iii) {\sc Imp.} & (iv) {\sc TT} & (v) {\sc WSRT}  \\
\midrule
\knnc           & 0.55            & 39/2/30    & 0.35\% & 5.30E-2 & 7.63E-2  \\
\onennc         & 0.66            & 47/2/22    & 2.36\% & 1.55E-5 & 2.81E-5  \\
\sbfc           & 0.99            & 70/0/1     & 25.4\% & $<$1E-8 & $<$1E-8  \\
\midrule
\lr             & 0.62            & 44/1/26    & 1.80\% & 4.06E-2 & 7.68E-3  \\
\mlpc           & 0.75            & 53/0/18    & 6.17\% & 7.00E-8 & 1.00E-8  \\
\bottomrule
\end{tabular}
}
\end{center}
\end{table}

Results similar in form to the results in the main paper (Figure~\ref{fig:flynn-base}, Table~\ref{tab:openml-ind-comp}) are presented in Figure~\ref{afig:flynn-base} and Table~\ref{atab:openml-ind-comp}. The results for \knnc, \onennc and \sbfc are repeated for ease of comparison. \fnn performs comparably to \lr (with a single epoch), outperforming it on 62\% of the datasets, and providing a 1.80\% median improvement with $p$-values of 0.0406 (TT) and 0.00768 (WSRT), indicating significant difference. When compared to \mlpc (with a single epoch), \fnn has a much more favorable performance, outperforming on 75\% of the datasets, and providing a 6.17\% median improvement and $p$-values of the order of $10^{-8}$, indicating significant improvement. This highlights the ability of \fnn to match or outperform gradient-descent-based methods (which are widely studied in the FL setting) when compared with some form of parity in the communication overhead across over 70 datasets from OpenML.

However, this comparison should be considered with the understanding that comparing inherently different machine learning methods can involve various caveats. In this comparison, we tried to explicitly discuss the choices we made and why we made them.

\subsection{Empirical evaluation of Differentially Private {\fnn}FL} \label{asec:emp-eval:dp}

To study the effect of $(\epsilon, 0)$-differential privacy on the classification performance of {\fnn}FL, we consider two sets of experiment. In one set of experiments, we consider the synthetic dataset (described in supplement \S \ref{asec:emp-eval:syndata}), while in another, we consider the a binary classification version of the MNIST dataset, where we are trying to classify between the (hard-to-classify) digits 3 and 8.

\paragraph{Synthetic data.}
In this set of experiments, we generate classification data in $\R^d$ with $d = 30$ for a 2-class classification problem with 5 modes per class. We create two training datasets, one with $n = 10^4$ samples and another with $n = 10^5$ samples to study the effect of increasing the number of samples. To evaluate the performance, we utilize a heldout test set of size $10^3$ samples.

We consider 4 sets of hyper-parameters for \fnn: 
\begin{enumerate}
    \item $m = 300, s = 3, \rho = 15, \gamma = 0.9$,
    \item $m = 300, s = 3, \rho=30, \gamma = 0.9$,
    \item $m = 600, s = 3, \rho=15, \gamma = 0.9$,
    \item $m = 600, s = 3, \rho=30, \gamma = 0.9$.
\end{enumerate}

We consider a 2-party problem ($\tau = 2$) where the training data is evenly between the parties, each with $\nicefrac{n}{2}$ samples.
For each hyper-parameter, we first invoke Train{\fnn}FLDP with {\tt IS\_DP = false} and record the non-DP test error. Then, we invoke Train{\fnn}FLDP with DP enabled (that is {\tt IS\_DP = true}) with different values of $\epsilon \in [0.25, 2.0]$ and $T \in [4, 600]$. The test error for each pair of $(\epsilon, T)$ is noted -- we perform 10 repetitions for each pair of $(\epsilon, T)$ and report the mean and the confidence interval of the obtained test error. The results are presented in Figure~\ref{afig:dp-all-trials-eps-linscale}. The results in Figure~\ref{fig:dp-eps-T} presents one of these results for $n = 10^5$.

\begin{figure}[t]
\centering
\includegraphics[width=\textwidth]{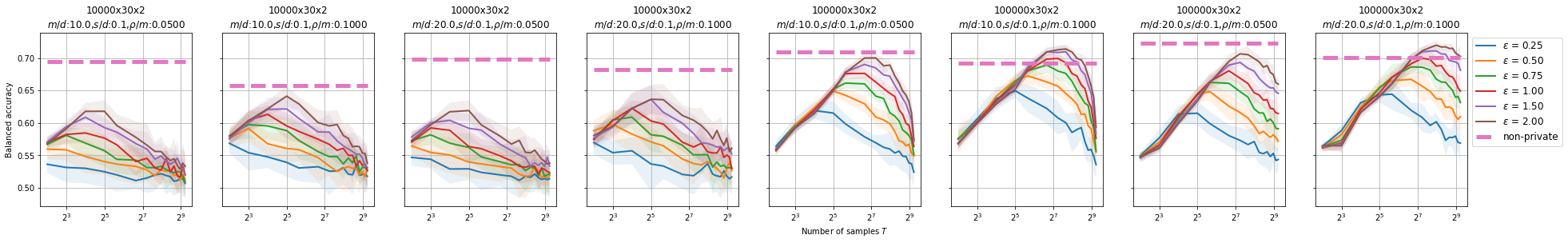}
\caption{Effect of DP on the {\fnn}FL accuracy for varying values of $\epsilon \in [0.25, 2]$ and $T \in [4, 600]$. The horizontal axis corresponds to the number of samples $T$ while the vertical axis corresponds to balanced classification accuracy (higher is better). In each of the plots, the horizontal dashed line corresponds to the accuracy of the non-private {\fnn}FL (Train{\fnn}FLDP invoked with {\tt IS\_DP = false}). Each of the lines correspond to a particular value of $\epsilon$, and plots the mean accuracy over 10 repetitions with the confidence interval in the form of translucent ribbons around the lines. The 4 left plots correspond to the 4 \fnn hyper-parameters with $n = 10^4$ samples, while the 4 right plots correspond to the 4 \fnn hyper-parameters with $n = 10^5$ samples.}
\label{afig:dp-all-trials-eps-linscale}
\end{figure}

The results indicate the intuitive behaviour that, as the number of samples $T$ increases, the {\fnn}FL performance improves up until a point and then it starts dropping. This behaviour is seen for all values of $\epsilon$ and all different hyper-parameters $(m, s, \rho, \gamma)$ of \fnn. Comparing the 4 left plots (with $n = 10^4$ samples) to the 4 right plots (with $n = 10^5$ samples), we see that, with increasing number of points, the non-private test accuracy is not significantly affected, but the performance of the DP {\fnn}FL improves significantly, and is able to almost match the non-private accuracy for all the \fnn hyper-parameters for moderately high $\epsilon \in [1, 2]$.

Figure~\ref{afig:dp-all-trials-eps-vs-acc} presents the dependence of the {\fnn}FL accuracy on $\epsilon$ more explicitly. The experimental setup is the same as Figure~\ref{afig:dp-all-trials-eps-linscale}, and, for each value of $\epsilon$, we select the $T$ with the best accuracy on the heldout set. We see the expected behaviour with the accuracy increasing with $\epsilon$.  The results in Figure~\ref{fig:dp-eps-v-acc} presents one of these results for $n = 10^5$.

\begin{figure}[bt]
\centering
\includegraphics[width=\textwidth]{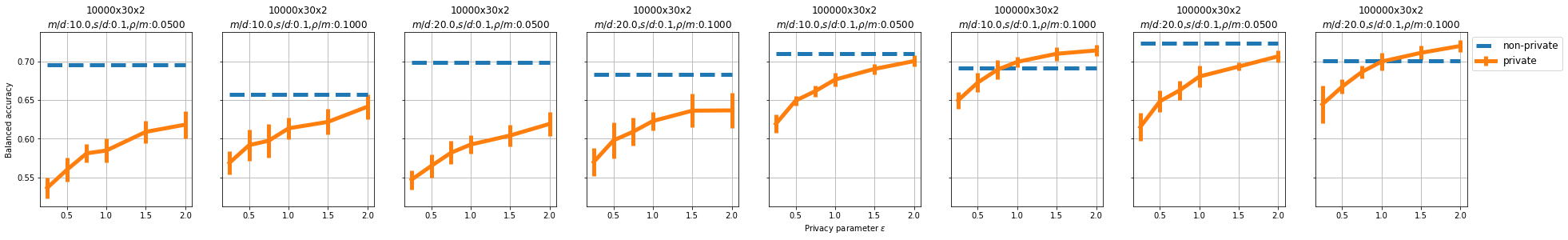}
\caption{Effect of DP on the {\fnn}FL accuracy for varying values of $\epsilon \in [0.25, 2]$. The horizontal axis corresponds to the privacy level $\epsilon$, while the vertical axis corresponds to balanced classification accuracy (higher is better). In each of the plots, the blue horizontal dashed line corresponds to the accuracy of the non-private {\fnn}FL (Train{\fnn}FLDP invoked with {\tt IS\_DP = false}). In each of the plots, the orange line corresponds to the best (over all values of $T$) mean $\pm$ standard deviation accuracy (aggregated over 10 repetitions) of the DP-enabled {\fnn}FL with increasing $\epsilon$. The 4 left plots correspond to the 4 \fnn hyper-parameters with $n = 10^4$ samples, while the 4 right plots correspond to the 4 \fnn hyper-parameters with $n = 10^5$ samples.}
\label{afig:dp-all-trials-eps-vs-acc}
\end{figure}

\paragraph{MNIST 3 vs 8 binary classification.}
We also evaluate the effect of DP on MNIST. For the binary 3 vs. 8 classification problem, $n \approx 12000$ and $d = 784$. We select the following \fnn hyper-parameters: $m/d = 30, s/d = 0.00625, \rho = 0.1 m, \gamma = 0.9$. We perform a similar experiment as with synthetic data and compare the performance of non-private Train{\fnn}FLDP with 2 parties ($\tau = 2$) to that of the DP-enabled Train{\fnn}FLDP with different values of $\epsilon \in [0.25, 2]$ and $T \in [4, 600]$. The results are presented in Figure~\ref{afig:dp-mnist}. The results in Figure~\ref{afig:dp-mnist3v8-all-trials} presents the aforementioned trend where the accuracy of DP {\fnn}FL improves with the number of samples $T$ up until a point at which it drops. Figure~\ref{afig:dp-mnist3v8-all-trials-eps-vs-acc} presents the expected trend of the improving performance of DP {\fnn}FL with increasing $\epsilon$.

\begin{figure}
\centering
\begin{subfigure}{0.45\textwidth}
\includegraphics[width=\textwidth]{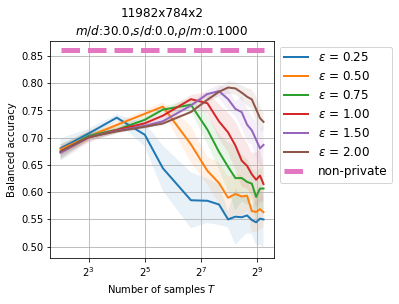}
\caption{$T$ vs accuracy for varying values of $\epsilon$.}
\label{afig:dp-mnist3v8-all-trials}
\end{subfigure}
~
\begin{subfigure}{0.45\textwidth}
\includegraphics[width=\textwidth]{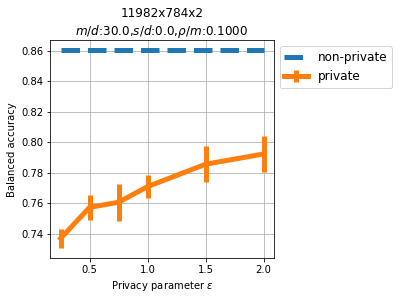}
\caption{$\epsilon$ vs accuracy}
\label{afig:dp-mnist3v8-all-trials-eps-vs-acc}
\end{subfigure}
\caption{Effect of the {\fnn}FL accuracy for varying values of $\epsilon \in [0.25, 2]$ and $T \in [4, 600]$ with MNIST dataset for a binary classification problem with labels $3$ and $8$. In figure~\ref{afig:dp-mnist3v8-all-trials}, the horizontal axis corresponds to the number of samples $T$ while the vertical axis corresponds to balanced classification accuracy (higher is better); the horizontal dashed line corresponds to the accuracy of the non-private {\fnn}FL (Train{\fnn}FLDP invoked with {\tt IS\_DP = false}), and each of the lines correspond to a particular value of $\epsilon$, and plots the mean accuracy over 10 repetitions with the confidence interval in the form of translucent ribbons around the lines. In figure~\ref{afig:dp-mnist3v8-all-trials-eps-vs-acc}, the horizontal axis corresponds to the privacy level $\epsilon$, while the vertical axis corresponds to balanced classification accuracy (higher is better); the blue horizontal dashed line corresponds to the accuracy of the non-private {\fnn}FL (Train{\fnn}FLDP invoked with {\tt IS\_DP = false}), and the orange line corresponds to the best (over all values of $T$) mean $\pm$ standard deviation accuracy (aggregated over 10 repetitions) of the DP-enabled {\fnn}FL with increasing $\epsilon$.}
\label{afig:dp-mnist}
\end{figure}

\end{document}